\PassOptionsToPackage{table}{xcolor}
\documentclass{artificer}
\usepackage{amsmath,amssymb,amsfonts}
\usepackage{marvosym}
\usepackage{pifont}
\usepackage{multirow}

\usepackage{amsmath,amsfonts,bm}

\def\eqref#1{equation~\ref{#1}}
\def\Eqref#1{Equation~\ref{#1}}

\def\1{\bm{1}}

\DeclareMathAlphabet{\mathsfit}{\encodingdefault}{\sfdefault}{m}{sl}
\SetMathAlphabet{\mathsfit}{bold}{\encodingdefault}{\sfdefault}{bx}{n}

\renewcommand{\Eqref}[1]{Eq.~(\ref{#1})}

\definecolor{AscentIndigo}{HTML}{8800C2}
\definecolor{AscentTint}{HTML}{EEEDFA}
\definecolor{AscentTintLight}{HTML}{F7F6FC}
\newcommand{\ascentmethod}[1]{\textbf{#1}}

\usepackage{url}
\usepackage{makecell}
\usepackage{amsthm}
\usepackage{float}
\usepackage{wrapfig}
\usepackage{flafter}
\usepackage{enumitem}
\usepackage{needspace}
\newtheorem{theorem}{Theorem}[section]
\newtheorem{lemma}[theorem]{Lemma}

\theoremstyle{definition}

\title{\raisebox{-0.12em}{\includegraphics[height=0.95em]{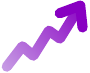}}\,ASCENT: Online Test-Time Training of Long-Horizon Agents via Self-Distillation of Verified Experience}

\author{Haodong Lu}
\author{Dong Gong}
\affiliation{University of New South Wales (UNSW Sydney)}
\emails{\email{haodong.lu@unsw.edu.au}, \email{dong.gong@unsw.edu.au}}

\abstract{
A large language model (LLM) agent solves long-horizon tasks through many
reasoning-action turns, with one verification signal at termination.
Deployed agents face streams of related tasks, making their own trajectories
a natural resource for improvement. In-context adaptation agents store
reflections, memories, or skills as text at the harness level, so reuse depends on retrieving the
right experience and on a frozen policy executing it. We study
\emph{Online Agentic Test-Time Training} (OaTTT), which trains the LLM's weights on the agent's own execution trajectories during
deployment. The agent executes each task once in one pass over the task stream,
and the executed trajectory with its verification result is the only and immediate learning
signal for weight updates that persist across tasks. Directly imitating or
reinforcing the generated tokens of this single attempt destabilizes the policy. We introduce
ASCENT (\emph{Agentic Self-distillation for Cross-task EvolutioN at Test-time}),
which instead self-distills verified experience.
A stable version of the LLM, its frozen initial copy, receives the verifier-accepted trajectory as privileged information and predicts next-token distributions along it with this hindsight. Distilling them into persistent LoRA fast weights updates the agent for later tasks, without an external reference solution or stronger teacher.
By further removing invalid-action turns from the trajectories, ASCENT can distill
enhanced privileged experience for more effective and efficient execution.
We characterize the population target of this distillation and the limits of sparse outcome selection.
Across ALFWorld, WebShop, and AppWorld on varied model scales, ASCENT improves task success and interaction
efficiency as experience accumulates, outperforms the compared online
adaptation methods, and transfers to held-out scenes, showing that a deployed
agent can consolidate its own verified experience into its policy without a
separate training phase or memory retrieval at inference.
}

\metadata[Project Page]{\href{https://artificer-ai-lab.github.io/ASCENT/}{Artificer-AI-Lab/ASCENT}}

\begin{document}
\raggedbottom

\maketitle

\section{Introduction}

A large language model (LLM) agent interacts with its environment to
accomplish a task, iteratively reasoning about the current state and taking
actions such as tool calls~\citep{yao2022react}. The LLM serves as the core
of the agent's decision policy. Agent tasks are often long-horizon, requiring multi-turn interactions with complex environments. Each completed
interaction forms one episode, a trajectory of reasoning, actions, and
observations.

Deployed agents often operate in persistent, long-running roles,
encountering a continuous stream of tasks and interactions with users and
environments. Because real-world tasks are complex and diverse, an agent built
on a frozen LLM may not handle every situation it faces. The
trajectories the agent gathers in the task stream are a natural resource for improving the agent at test time, as experience from earlier tasks can enhance performance on related tasks encountered later.

Relying on the in-context learning capability of LLMs, one line of work
stores reflections on completed episodes, memories, and reusable skills
as text at the harness level and reuses them in later
interactions~\citep{shinn2023reflexion,zhao2023expel}. Online variants
accumulate this experience across tasks during
deployment~\citep{ouyang2025reasoningbank,zhang2025ace,fang2025memp,zhang2026memrl}.
Such in-context adaptation requires no parameter updates and changes
only the model's input, keeping the model itself frozen. Its gains can therefore be
limited by how well relevant experience or skill is retrieved and how reliably
the frozen policy executes it over a long interaction.
Efforts are thus also taken to improve the core LLM for the agentic task process, via parametric adaptation on model weights.
Some test-time training (TTT) methods adapt the LLM's parameters within a
single episode and reset them afterward, training on self-generated
supervision (aTTT~\citep{wang2026attt}) or
QA pairs from the episode's accumulated context (TMEM~\citep{ren2026tmem}).
Other TTT methods require pre-deployment training of the adaptation
behavior~\citep{zweiger2025seal} or multiple sampled
attempts per problem~\citep{zuo2025ttrl,wang2026ttpo}.
Reinforcement learning (RL) post-training also improves the agentic
capability of LLMs, which optimizes the policy before deployment on a
training task set, typically with many attempts per
task~\citep{shao2024deepseekmath,feng2025gigpo,tan2026hcapo,wu2026seed,lu2026sdar} or single-attempt RL with a learned critic~\citep{hou2026sao}.
Thus, none of them adapts a long-horizon agent's weights at test time, in a single pass over the task stream with one attempt per task, while carrying updates across tasks.

\Needspace{8\baselineskip}
To explore the potential of evolving LLMs through test-time training in
deployment, we study \emph{Online Agentic Test-Time Training} (OaTTT), which
trains the LLM's weights directly on the agent's own execution trajectories to
improve task success.
Tasks arrive in a single-pass stream, and the agent executes the task once, receiving verification from environment.
Each executed trajectory and its verification result drive weight updates that persist across following tasks.
Like in-context adaptation agents~\citep{ouyang2025reasoningbank,zhang2025ace}, the agent learns from each trajectory once, without multiple trails. It accumulates learning in LLM weights, whereas those agents use textual harness-level memory. Learning occurs entirely at test time, without task-specific or offline training. We ask:

\begin{center}
  \begin{minipage}{0.7\linewidth}
  \centering
  \emph{{Can a deployed agent turn its own verified execution trajectories,
  one per task, into persistent weight updates that improve its success on
  subsequent tasks?}}
  \end{minipage}
\end{center}

  We propose ASCENT (\emph{Agentic Self-distillation for Cross-task
EvolutioN at Test-time}), which enables agents to learn during deployment, turning each verified experience trajectory into a persistent weight update.
Training directly on these streaming trajectories is straightforward but
fragile, because each task yields only one attempt by the current policy.
Directly imitating or reinforcing the verified trajectories, with direct imitation or single-attempt policy gradient (Sec.~\ref{sec:direct_ntp}), destabilizes the
policy, erodes executable behavior (Fig.~\ref{fig:ntp_collapse}), and lowers
task success.
ASCENT instead \emph{self-distills from the verified experience} to enable stable online test-time training.
A stable version of the LLM, its frozen initial copy, receives the
verifier-accepted trajectory as privileged information and, at every token
position of that trajectory, predicts a full next-token distribution with this
hindsight.
Environment verification acts as rejection sampling, providing self-distillation positive guidance over the full horizon.
Distilling it into persistent LoRA fast weights absorbs selected experience without treating every generated token as a direct target~\citep{hu2021lora,agarwal2023gkd,zhao2026opsd}.
ASCENT thus distills hindsight from verified experience into the LLM weights, using this stable version of the LLM as the teacher for stable training.
Beyond raw verified experiences, ASCENT can further refine accepted trajectories by removing invalid-action \emph{turns}, while retaining the reasoning-action content of valid turns as privileged information. Filtering changes only the teacher input. The agent is still trained at removed-turn positions.
The filtered trajectories provide the agent with more effective and efficient execution paths.
The execution of the next task is based on the updated agent, enabling verification-gated online self-evolution. Across tasks and model scales, ASCENT improves task success and interaction efficiency as experience accumulates, while transferring learned weights to held-out scenes in the stream (Fig.~\ref{fig:ascent_overview}).
Our contributions are as follows:
\begin{itemize}[leftmargin=*]
    \item We study OaTTT, in which
    a deployed agent executes each task once in one pass over its task
    stream, and the executed trajectory with its verification result is the only
    data for weight updates that persist across tasks. In OaTTT, we identify the instability of directly imitating or reinforcing these trajectories.
    \item We propose ASCENT, which self-distills verified experience for
    stable online test-time training. A stable version of the LLM, its frozen initial copy, conditioned on
    the verifier-accepted trajectory as privileged information, with invalid-action turns further removed,
    supplies hindsight-informed full-distribution targets that are distilled into persistent LoRA
    fast weights. We characterize the population target of this distillation
    and the limits of sparse outcome selection.
    \item We validated ASCENT across different types of tasks/steams, across  ALFWorld, WebShop, and AppWorld on varied model scales.
    ASCENT successfully improves task success and interaction efficiency with experience, outperforms the compared online adaptation baselines, and transfers to held-out scenes with effective continued test-time training. The influence of prevailed information and distillation divergence are investegated. We also
    analyzed in-context adaptation failure modes, and
    present a preliminary in-context and parametric co-evolution study with ASCENT.
\end{itemize}

\section{Related Work}
\label{sec:related_work}

\noindent\textbf{Self-evolving agents.}
In-context methods store reflections, memories, or reusable skills and retrieve
them at
inference~\citep{shinn2023reflexion,zhao2023expel,wang2023voyager,
suzgun2025dynamiccheatsheet,ouyang2025reasoningbank,zhang2025ace,fang2025memp,
xu2025amem,zhang2026memrl}, so reuse depends on retrieval and on the frozen policy
executing it over a long interaction. Parametric methods internalize experience
in weights~\citep{zweiger2025seal,ren2026tmem}. ASCENT consolidates
verifier-selected interactions into persistent LoRA fast weights that carry
across tasks.

\noindent\textbf{Online learning and test-time training.}
Online learning updates a model from sequentially arriving
data~\citep{shalev2012online}. Test-time training adapts parameters on
inference-time
signals~\citep{sun2019ttt,wang2020tent,wang2022cotta,hardt2023ttnn,hu2025ttl},
though Agentic TTT and TT-SI reset parameters after each episode or
instance~\citep{wang2026attt,acikgoz2025selfimprovingagents}. Existing
LLM self-training and RL methods require multiple solutions,
correctness-selected generations, or pre-deployment
attempts~\citep{zuo2025ttrl,wang2026ttpo,
zelikman2022star,yuan2023rft,gulcehre2023rest,singh2023restem,
shao2024deepseekmath,kim2026whale,hou2026sao}. ASCENT persists updates across
tasks from one attempt each, selecting by the environment's verification outcome and
distilling frozen-teacher soft targets.

\noindent\textbf{On-policy distillation and self-distillation.}
On-policy distillation queries a teacher on student-sampled sequences,
reducing train--generation
mismatch~\citep{agarwal2023gkd}. OPSD conditions a self-teacher on reference
solutions while the student receives only the
problem~\citep{zhao2026opsd}. SDFT conditions the teacher on demonstrations,
and SDPO on environment feedback or the model's own successful
attempts~\citep{shenfeld2026sdft,hubotter2026sdpo}. Recent agentic methods
use privileged or hindsight-conditioned teacher guidance, mostly before
deployment and often combined with
RL~\citep{wang2026skillsd,lu2026sdar,wu2026seed,wang2026agentopsd,liu2026smcsd,chen2026hintsd}.
CLaaS studies online self-distillation with asynchronous
replay~\citep{yang2026claas}.
Like OPSD, ASCENT keeps the teacher fixed at the initial model and
distills at test time, without replay, immediately after each
verifier-accepted interaction of the student.

\begin{figure}[t!]
\centering
\includegraphics[width=\linewidth]{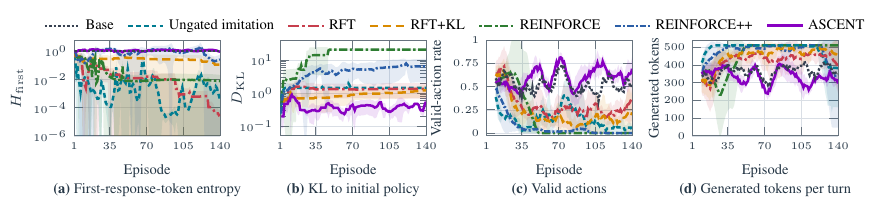}
\caption{\textbf{Imitation and reinforcement lose executable behavior.} Online
per-episode updates (Sec.~\ref{sec:direct_ntp}) of Qwen3.5-4B on the
ALFWorld seen stream, with std shading: (a) first-response-token entropy, (b) KL
to the initial policy, (c) valid-action rate, and (d) generated tokens per turn.}
\label{fig:ntp_collapse}
\end{figure}

\section{Method}
\label{sec:method}

\subsection{Online Agentic Test-Time Training (OaTTT)}
\label{sec:method_setting}

\noindent\textbf{Self-evolving agent-environment interaction.}
At deployment, the agent encounters a one-pass task stream
$\mathcal X_{1:N}=(x_1,\ldots,x_N)$, where each $x_i$ specifies a goal in a
partially observed interactive environment $\mathcal E$. We call the
agent's single attempt at task $x_i$ episode $i$. At turn $t$ of episode $i$,
the agent conditions on
\begin{equation}
H_{i,t}=(x_i,o_{i,1},w_{i,1},o_{i,2},\ldots,
w_{i,t-1},o_{i,t}),
\end{equation}
where $o_{i,t}$ is the current observation and $w_{i,t}$ is the generated
reasoning-action response~\citep{yao2022react}. An action parser $\rho$
extracts the action $a_{i,t}=\rho(w_{i,t})$ from each response, and no action
results when parsing fails.
Let $\mathcal M_i$ denote an optional,
evolving cross-task-episode memory, such as reflections, memories, or skills
accumulated from earlier
episodes~\citep{ouyang2025reasoningbank,zhang2025ace,fang2025memp}. The deployed
policy is
\begin{equation}
\pi_i(\cdot\mid H_{i,t})
:=p_{\theta_i}\bigl(\cdot\mid g(H_{i,t};\mathcal M_i)\bigr),
\qquad
w_{i,t}\sim\pi_i(\cdot\mid H_{i,t}).
\label{eq:behavior_policy}
\end{equation}
Here $g$ renders the instructions, history, and content from $\mathcal M_i$. The environment executes $a_{i,t}$ and returns
$o_{i,t+1}\sim\mathcal E(\cdot\mid H_{i,t},a_{i,t})$, and episode $i$ yields its
complete trajectory $\tau_i$ at termination.

\noindent\textbf{OaTTT as agent evolution.}
In online evolution, there is no separate task-specific training split or
pre-deployment adaptation. We aim to update the policy LLM model from its own
interaction experience during deployment. The current policy $\pi_i$ makes a single
attempt at each task $x_i$, producing one episode:
\begin{equation}
\tau_i\sim\mathbb P_{\pi_i,\mathcal E}(\cdot\mid x_i),
\qquad
v_i=\mathcal V(x_i,\tau_i)\in\{0,1\},
\label{eq:online_ttt_rollout}
\end{equation}
where $v_i$ is the sparse, episode-level verification outcome, returned when
episode $i$ ends and recorded before any update. The trajectory $\tau_i$ and
outcome $v_i$ then update the policy:
\begin{equation}
    (\theta_{i+1},\mathcal M_{i+1})=
\mathsf{Evolve}(\theta_i,\mathcal M_i;x_i,\tau_i,v_i).
\label{eq:state_evolution}
\end{equation}
The update operator $\mathsf{Evolve}$ is specific to each method.
In-context adaptation methods evolve the agent by updating its memory
$\mathcal M_i$ while keeping the model fixed. Because the policy in
\Eqref{eq:behavior_policy} conditions on $\mathcal M_i$, this memory update
also changes the policy. Parametric methods evolve
$\theta_i$~\citep{zweiger2025seal}, joint methods may evolve
both, and frozen methods evolve neither.
This ordering follows a prequential online
protocol~\citep{dawid1984prequential,blum1999progressive}: each reported
outcome is computed before its own trajectory can influence the policy.
In OaTTT, that scored trajectory and its verification result are also
the only learning data, with no retries or additional attempts. No replay
buffer is used in this paper.
Reported success follows each benchmark's criterion, which can be stricter
than the one defining $v_i$. Success on later tasks thus reflects how much
earlier experience helps. Observations describe the resulting states. They give
no turn-level reward or causal credit for the verified
outcome~\citep{sutton2018reinforcement,zeng2025turncredit,tan2026hcapo}.
We call this protocol the OaTTT protocol, and all online methods, in-context or
parametric, follow it. In OaTTT, the LLM's weights adapt at test time, in one
pass over the task stream, from a single multi-turn attempt per task and its
sparse episode-level verification, and the updated weights carry across tasks.

\Needspace{10\baselineskip}
\subsection{OaTTT via Direct Imitation or Reinforcement of Experience}
\label{sec:direct_ntp}

A straightforward way to perform OaTTT is to directly imitate or reinforce the
agent's own interaction trajectories. We formulate five such straightforward baseline methods in two families:
direct imitation of the generated tokens
(Online ungated imitation, Online RFT, and Online RFT+KL), and single-attempt
policy gradient, which reinforces or suppresses them with a signed advantage. Since
group-relative advantages~\citep{shao2024deepseekmath} need several attempts per
task, Online REINFORCE uses REINFORCE~\citep{williams1992reinforce} with a
cross-task baseline, and Online REINFORCE++ adapts
REINFORCE++~\citep{hu2025reinforcepp} to normalize advantages over the stream.

\begin{figure}[t!]
\centering
\includegraphics[width=\textwidth]{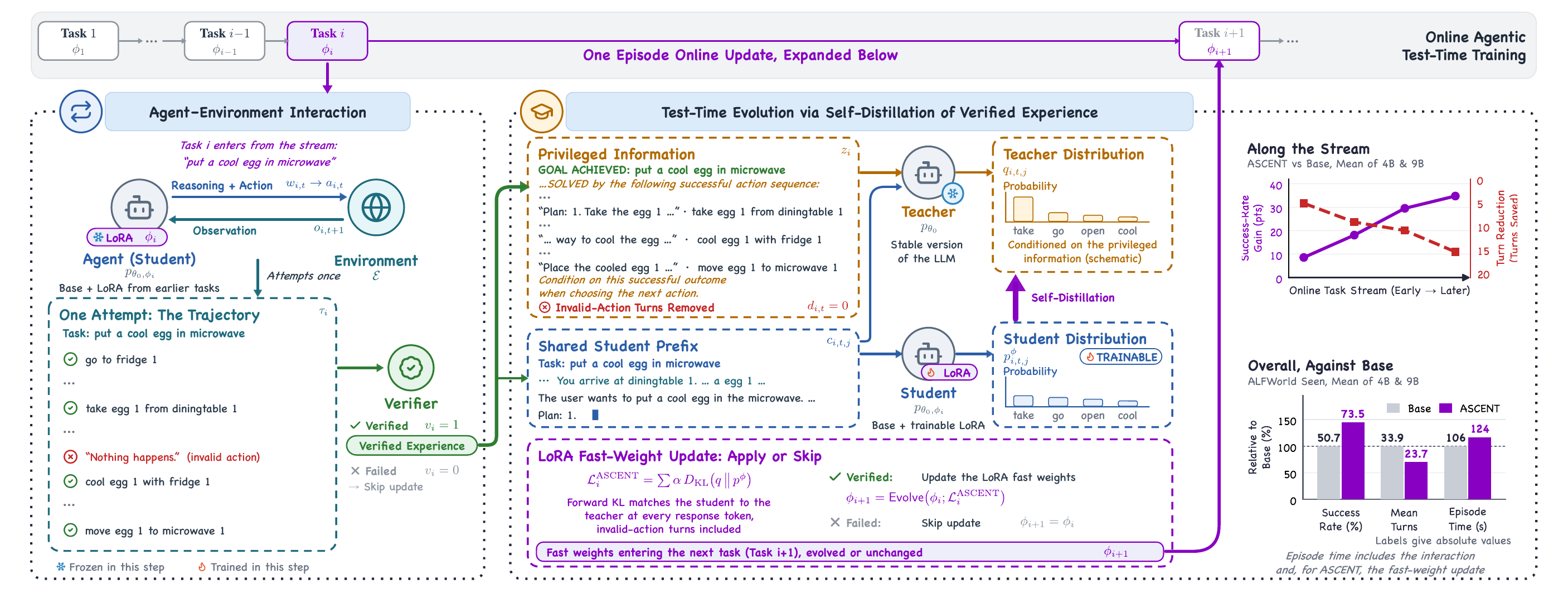}
\caption{\textbf{Online Agentic Test-Time Training of ASCENT.} The student
makes a single attempt per task. A verified trajectory becomes privileged information
for the frozen teacher, whose soft targets update the LoRA fast
weights for the next task. Charts on the right show ALFWorld seen-stream means of 4B and
9B runs (per-scale in Fig.~\ref{fig:excess_wins_intro}(a) and
Table~\ref{tab:alfworld_seen}). Episode time includes the update.}
\label{fig:ascent_overview}
\end{figure}

\noindent\textbf{Imitation and reinforcement on the generated trajectories.}
In a controlled study of OaTTT, we examine the learning dynamics of these five
methods along the task stream. For the agent being evolved, every method uses fast-weight LoRA
updates~\citep{hu2021lora}: $\theta_0$ stays frozen and a LoRA adapter is
carried across episodes, so the policy entering episode $i$ is
$p_{\theta_0,\phi_i}$. Let
$y_{i,t,1:L_{i,t}}$ be the $L_{i,t}$ tokens of response $w_{i,t}$ and
$\widehat H_{i,t}=g(H_{i,t};\mathcal M_i)$ the rendered context of the attempt,
including any retrieved context. The student prefix
$c_{i,t,j}=\widehat H_{i,t}\mathbin{\oplus}y_{i,t,<j}$ is taken from the single
attempt at $x_i$ and reused unchanged during the update. Let $\mathcal I_i$ index
the turns with nonempty responses, and let
$\alpha_{i,t,j}=(|\mathcal I_i|L_{i,t})^{-1}$ average tokens within each turn
and then across turns. Starting from $\phi=\phi_i$, all five methods update the
adapter $\phi$ with a common loss that puts a signed weight $\omega_{i,t,j}$ on
each generated token:
\begin{equation}
\mathcal L_i(\phi)
=
-\sum_{t\in\mathcal I_i}\sum_{j=1}^{L_{i,t}}
\alpha_{i,t,j}\,\omega_{i,t,j}
\log p_{\theta_0,\phi}\!\left(y_{i,t,j}\mid c_{i,t,j}\right).
\label{eq:ntp_ttt}
\end{equation}
We detail the online per-episode update of each method used in this study:
\begin{enumerate}[leftmargin=*,itemsep=1pt,topsep=2pt]
\item \textbf{Online ungated imitation}, as a naive baseline, sets $\omega_{i,t,j}=1$, so every trajectory, including failed ones, becomes a target.
\item \textbf{Online RFT}~\citep{yuan2023rft,shao2024deepseekmath} sets
$\omega_{i,t,j}=v_i$, so only verifier-accepted trajectories become training
targets and failed episodes contribute no gradient.
\item \textbf{Online RFT+KL} adds a weighted
$D_{\mathrm{KL}}(p_{\theta_0}(\cdot\mid c_{i,t,j})\,\|\,p_{\theta_0,\phi}(\cdot\mid c_{i,t,j}))$
to each term of Online RFT, anchoring the policy to the frozen base. The other
methods add no KL term to the loss.
\item \textbf{Online REINFORCE} sets $\omega_{i,t,j}=v_i-\bar v_{\le i}$, with
$\bar v_{\le i}$ the running mean verification outcome up to episode $i$, so
after a first success, failed episodes lower the likelihood of their tokens.
\item \textbf{Online REINFORCE++} sets $\omega_{i,t,j}$ to a running-normalized
token advantage (reward $+1$ for success, $-1$ for failure, minus KL penalties
to the base from that token onward) and uses the clipped surrogate
\mbox{$\min(\rho\,\omega_{i,t,j},\operatorname{clip}(\rho,1\pm\epsilon)\,\omega_{i,t,j})$},
with $\rho$ the ratio to the policy $p_{\theta_0,\phi_i}$ that made the attempt.
\end{enumerate}

All five methods learn only from the tokens the agent generated. Their weights
make each of these tokens more or less likely. They say nothing about the other
tokens at that position. Lowering a token simply moves its probability to tokens
the policy already favors, and the KL anchor of Online RFT+KL pulls toward a
base model that has not seen the episode. Every update therefore stays centered
on the generated tokens, including incidental reasoning, formatting, and
detours.

\noindent\textbf{Direct online imitation or reinforcement degrades agent execution.}
An online update rule can become \emph{unstable} when, along the stream, its
policy drifts ever further from the initial model and degrades while its task success falls
below that of the base.
On the ALFWorld stream (Fig.~\ref{fig:ntp_collapse}), Online ungated imitation
and Online RFT become nearly deterministic at the first response token while
drifting from the initial model. Online RFT+KL drifts more slowly but still
substantially (a,b). For the policy-gradient methods, successes soon stop, so
nearly every update comes from a failed episode. Online REINFORCE ends
collapsed onto first response tokens that the base finds very unlikely, while
the first-response-token entropy of Online REINFORCE++ swings above and below
the base as it drifts (a,b). All five increasingly produce invalid actions (c),
mostly in long responses (d), and over the second half of the stream they
succeed less often than the base. These straightforward baselines degrades the agent policy and collapse to lower success rate than the un-trained model.

In OaTTT, we aim to enable single-attempt, per-sample online learning (batch size 1) of LLM weights, analogous to in-context adaptation based agent evolving \cite{ouyang2025reasoningbank,yao2022react}.
In the straightforward online learning baselines, each update relies on a single execution trajectory, so sample-specific characteristics and noise are not summarized or averaged. As subsequent attempts come from the updated policy, with the gradient from direct imitation and reinforcement the generated tokens, the issue can be amplified, causing more overfitting, policy drift, and eventual degradation.
ASCENT addresses this by distilling experience with an OPSD objective, using distributions from an anchored self-teacher conditioned on hindsight from verified trajectories as targets. 

\Needspace{6\baselineskip}
\subsection{ASCENT: OaTTT via Self-Distillation of Verified Experience}
\label{sec:ascent}

\noindent\textbf{Overview.}
ASCENT stably distills verified experience from hindsight-informed predictions of the fixed initial model $p_{\theta_0}$ into LoRA fast weights $\phi_i$~\citep{hu2021lora}, evolving the agent $p_{\theta_0,\phi_i}$ during OaTTT.
In episode $i$, the agent executes $x_i$ once to produce trajectory $\tau_i$ and verification $v_i$. If $v_i=1$, the current agent is distilled into $\phi_{i+1}$ from the frozen base model conditioned on privileged information $z_i$ from the filtered trajectory. Otherwise, the update is skipped.
The updated agent produces better trajectories
enabling online self-evolution along the stream
(Fig.~\ref{fig:ascent_overview}).

\subsubsection{On-Policy Self-Distillation Objective}
\label{sec:ascent_distillation}

On-policy distillation trains on student-generated prefixes using a
teacher's next-token distributions~\citep{agarwal2023gkd}. OPSD conditions a
self-teacher on privileged information unavailable to the
student~\citep{zhao2026opsd}. ASCENT applies this asymmetry to each
verifier-accepted episode, with the frozen base model $p_{\theta_0}$ as the
teacher.

The verification outcome $v_i$ selects trajectories for adaptation. It
assigns no credit to individual turns. Updating only on accepted episodes
($v_i=1$) acts as rejection sampling~\citep{yuan2023rft,shao2024deepseekmath},
so the teacher's hindsight always comes from a trajectory that passed
verification. With binary rewards, maximizing successful-trajectory likelihood gives an unbiased policy-gradient estimate. ASCENT thus learns only from verified successes, requiring no value model or critic under sparse rewards.
From the trajectory, without an external reference solution,
we construct the privileged information $z_i$
(Sec.~\ref{sec:ascent_privileged}) and obtain the following next-token
distributions at each shared student prefix $c_{i,t,j}$:
\begin{align}
q_{i,t,j}(\cdot)
&=
\operatorname{stopgrad}\!\left[
p_{\theta_0}\!\left(
\cdot\mid z_i\mathbin{\oplus}c_{i,t,j}\right)\right],
\label{eq:ascent_teacher}\\
p_{i,t,j}^{\phi}(\cdot)
&=
p_{\theta_0,\phi}\!\left(
\cdot\mid c_{i,t,j}\right).
\label{eq:ascent_student}
\end{align}
The turn set $\mathcal I_i$, token lengths $L_{i,t}$, and weights
$\alpha_{i,t,j}$ are those of \Eqref{eq:ntp_ttt}, so ASCENT trains at the same
student prefixes as direct imitation but replaces its target. Starting from
$\phi=\phi_i$, ASCENT minimizes the forward KL to the detached teacher
distributions:
\begin{equation}
\mathcal L_i^{\mathrm{ASCENT}}(\phi)
=
\sum_{t\in\mathcal I_i}\sum_{j=1}^{L_{i,t}}
\alpha_{i,t,j}
D_{\mathrm{KL}}\!\left(
q_{i,t,j}\,\middle\|\,p_{i,t,j}^{\phi}
\right).
\label{eq:ascent_objective}
\end{equation}
The fast weights entering episode $i+1$ are
\begin{equation}
\phi_{i+1}=
\begin{cases}
\mathsf{Evolve}(\phi_i;\mathcal L_i^{\mathrm{ASCENT}}), & v_i=1,\\
\phi_i, & v_i=0,
\end{cases}
\label{eq:ascent_update}
\end{equation}
where $\mathsf{Evolve}$ is the operator from \Eqref{eq:state_evolution},
here a few optimizer steps on \Eqref{eq:ascent_objective}
(Appendix~\ref{app:implementation}). The teacher stays fixed at $p_{\theta_0}$,
so self-distillation injects each verified trajectory's experience into the
student's fast weights through a stable version of the LLM. Student updates
never change the teacher's targets for a given trajectory, which helps
mitigate student drift during online learning. Because $z_i$ covers the full horizon of the verified trajectory,
the teacher sees at each turn how the episode continued to verified success,
giving the student better-informed turn-level guidance. We next describe how
the verified trajectory supplies this hindsight and the student
prefixes.

\subsubsection{Distilling Verified and Filtered Experience Trajectories}
\label{sec:ascent_privileged}

\noindent\textbf{Hindsight from the complete trajectory.}
Episode $i$, with $T_i$ turns, yields
\begin{equation}
\tau_i=\bigl(x_i,o_{i,1},(w_{i,t},o_{i,t+1})_{t=1}^{T_i}\bigr).
\label{eq:completed_interaction}
\end{equation}
It contains the task, each generated response, and the observation returned
after its action. At a student prefix $c_{i,t,j}$, the student had not yet
seen the rest of $w_{i,t}$, the later turns and observations, or the
verification outcome. We call this later part, together with its verified
outcome, \emph{hindsight}: it shows how the episode continued from each earlier
decision to verified success~\citep{harutyunyan2019hca,tan2026hcapo}.
Although it does not identify which actions caused success, it provides the student with the successful experience as privileged information via the self-teacher.

\noindent\textbf{Action-validity filter on the privileged information.}
Even an accepted trajectory can contain turns whose response yields no
parseable action or whose action the environment does not execute, such as an
unavailable command or code that raises an error. Invalid-action signals arise naturally during execution, requiring neither additional cost nor credit assignment. ASCENT further marks them
with a validity indicator read from the environment's response at each turn:
\begin{equation}
d_{i,t}=
\begin{cases}
1, & \text{if the environment executes the action $\rho(w_{i,t})$},\\
0, & \text{otherwise, including when no action is parsed}.
\end{cases}
\label{eq:admissibility_mask}
\end{equation}
The filter records only whether an action was executed. It does not judge
whether a valid action helps the task, gives no reward, and assigns no credit.

The teacher's privileged information is then
\begin{equation}
z_i=h_{\mathrm{ASCENT}}\!\left(
x_i,\,(w_{i,t})_{t:\,d_{i,t}=1}
\right),
\label{eq:privileged_record}
\end{equation}
where the fixed map $h_{\mathrm{ASCENT}}$ serializes the task goal, a
success statement, and the reasoning--action responses of valid turns,
with actions parsed by $\rho$. By default, $z_i$ omits observation text. The
verified trajectory thus plays two roles. Its valid turns build $z_i$,
while all its nonempty responses, including those of invalid-action turns,
supply the student prefixes $c_{i,t,j}$.
Sec.~\ref{sec:privileged_ablation} ablates the content of $z_i$ and the
filter.

\noindent\textbf{ASCENT distills experience online stably through full-vocabulary distribution matching.}
In comparison to the baseline methods of Sec.~\ref{sec:direct_ntp}, ASCENT updates the same student distribution $p^{\phi}$ at the same student prefixes, with a different target.
With position indices suppressed, for the generated token $y$ (one-hot vector $\mathbf e_y$) and student logits $\bm\ell^{\phi}$, a term of
\Eqref{eq:ntp_ttt} (of the baseline) and the forward KL of \Eqref{eq:ascent_objective} (of ASCENT) have
gradients
\begin{equation}
\nabla_{\bm\ell^{\phi}}\bigl[-\omega\log p^{\phi}_{y}\bigr]
=\omega\,\bigl(p^{\phi}-\mathbf e_y\bigr),
~~~~~ \text{and}~~~~~
\nabla_{\bm\ell^{\phi}}D_{\mathrm{KL}}\bigl(q\,\|\,p^{\phi}\bigr)
=p^{\phi}-q,~~~ \text{respectively}.
\label{eq:soft_vs_onehot_gradient}
\end{equation}
Minimizing the first term pushes $y$'s logit by $\omega(1-p^{\phi}_y)$ and every
other logit $k$ by $-\omega p^{\phi}_k$, in proportion to its current probability. The methods of
Sec.~\ref{sec:direct_ntp} thus perform hard-target matching on the generated
token, where a positive weight (direct imitation, and policy gradient on successful
episodes) makes the student more confident in $y$, a negative weight (policy
gradient, mainly on failed episodes) hands $y$'s probability to tokens the
student already favors, and both tend to sharpen the distribution around the
student's own outputs.
Minimizing the second in ASCENT pushes each logit $k$ by
$q_k-p^{\phi}_k$ and stops only when $p^{\phi}=q$.
ASCENT thus performs full-vocabulary distribution matching, so alternatives the hindsight teacher prefers rise, and $y$ falls
when the teacher disagrees, for example at an invalid-action turn omitted from
$z_i$.
As a stable LLM copy (i.e., the initial checkpoint), the teacher does not track the student, preventing the target from sharpening with it. While a stable yet refreshed teacher could better follow the evolving policy with a more precisely balanced trade-off, we find the frozen initial model (the default in \cite{zhao2026opsd}), simple and effective for online streaming adaptation (Appendix~\ref{app:teacher_dynamics}).
In Fig.~\ref{fig:ntp_collapse}, ASCENT stays closest to the initial
model while keeping its valid-action rate, with a relatively more stable dynamic behavior.
Under a fixed collection distribution, the population target averages the teacher distributions of accepted trajectories through each prefix (Appendix~\ref{app:theory_transfer}).

\noindent\textbf{Potential turn efficiency in long-horizon tasks.}
Under a turn budget, an unnecessary turn uses up part of the budget, so a
response that starts a detour can be less likely to reach verified success.
Verifier selection gives such responses less weight
(Appendix~\ref{app:theory_selection}), and the hindsight teacher, which sees how
the accepted episode reached success, can potentially lower them at earlier student
prefixes. We empirically show that ASCENT has the fewest mean turns in Tables~\ref{tab:alfworld_seen},
\ref{tab:main_webshop}, \ref{tab:alfworld_unseen_transfer},
and~\ref{tab:alfworld_unseen_continual}, and fewer turns than the base along
the stream (Fig.~\ref{fig:excess_wins_intro}(a)), including on tasks that both
the un-evolved base model and ASCENT solve (Sec.~\ref{sec:turn_efficiency}). Without
turn-level credit, an accepted trajectory can still contain unnecessary turns,
and one attempt cannot show what an alternative action would have achieved
(Appendix~\ref{app:theory_detours}). We leave a theoretical understanding and complete
proof of this mechanism for future work.

\section{Experiments}
\label{sec:experiments}

\begin{table*}[t!]
\centering
\small
\setlength{\tabcolsep}{3.4pt}
\renewcommand{\arraystretch}{1.05}
\begin{tabular}{@{}l cccccc c cc@{}}
\toprule
\textbf{Method} & Pick & Look & Clean & Heat & Cool & Pick2 &
\textbf{Avg. SR} & Turns & \makecell{Time\\{\scriptsize GPU\textperiodcentered h\,/\,s}} \\
\midrule
\rowcolor{gray!22}\multicolumn{10}{l}{\textit{Qwen3.5-4B}}\\
\midrule
Base & 91.4 & 46.2 & 40.7 & 6.2 & 8.0 & 54.2 & 46.4 & 35.1 & 97 \\
\midrule
\multicolumn{10}{l}{\textbf{Offline}}\\
\quad MemP & 94.3 & 53.8 & \textbf{77.8} & 0.0 & 36.0 & 66.7 & 61.4{\scriptsize$\pm$6.5} & 28.6 & 5.2\,/\,61 \\
\quad ACE & 85.7 & 46.2 & 70.4 & 6.2 & 60.0 & 58.3 & 60.7{\scriptsize$\pm$9.2} & 31.1 & 6.3\,/\,81 \\
\quad EvoSkill & 93.3 & 17.9 & 35.8 & 6.2 & 32.0 & 66.7 & 49.8{\scriptsize$\pm$4.9} & 34.4 & 5.9\,/\,123 \\
\quad GEPA & \textbf{100.0} & 23.1 & 22.2 & \textbf{37.5} & \textbf{64.0} & 66.7 & 58.6{\scriptsize$\pm$3.4} & 34.5 & 6.3\,/\,91 \\
\quad TextGrad & 88.6 & 46.2 & 22.2 & 6.2 & 12.0 & 50.0 & 42.1{\scriptsize$\pm$4.9} & 37.5 & 8.3\,/\,104 \\
\quad Trace2Skill & 94.3 & \textbf{69.2} & 29.6 & 0.0 & 16.0 & \textbf{75.0} & 51.4{\scriptsize$\pm$6.1} & 33.0 & 5.2\,/\,92 \\
\midrule
\multicolumn{10}{l}{\textbf{Online}}\\
\quad MemP & 94.3 & 61.5 & 71.6 & 12.5 & 36.0 & 73.6 & 63.6{\scriptsize$\pm$5.9} & 29.7 & 58 \\
\quad ReasoningBank & 94.3 & 30.8 & 33.3 & 14.6 & 22.7 & 61.1 & 49.0{\scriptsize$\pm$1.5} & 35.6 & 89 \\
\quad ACE & 90.5 & 17.9 & 49.4 & 22.9 & 42.7 & 41.7 & 51.2{\scriptsize$\pm$13.8} & 34.6 & 92 \\
\quad A-Mem & 88.6 & 30.8 & 35.8 & 10.4 & 20.0 & 58.3 & 46.7{\scriptsize$\pm$5.4} & 34.5 & 79 \\
\quad MemRL & 85.7 & 23.1 & 33.3 & 12.5 & 16.0 & 62.5 & 45.0{\scriptsize$\pm$5.2} & 35.4 & 89 \\
\quad Dynamic Cheatsheet & 91.4 & 23.1 & 44.4 & 6.2 & 28.0 & 54.2 & 48.6{\scriptsize$\pm$4.1} & 36.9 & 94 \\
\quad TT-OPSD & 95.2 & 46.2 & 29.6 & 2.1 & 25.3 & 72.2 & 51.0{\scriptsize$\pm$3.7} & 32.9 & 87 \\
\rowcolor{AscentTint}\quad \ascentmethod{ASCENT} & 99.0 & \textbf{69.2} & 69.1 & 20.8 & 54.7 & \textbf{75.0} & \textbf{69.5}{\scriptsize$\pm$5.6} & \textbf{26.3} & 115 \\
\midrule
\rowcolor{gray!22}\multicolumn{10}{l}{\textit{Qwen3.5-9B}}\\
\midrule
Base & 94.3 & 61.5 & 44.4 & 0.0 & 8.0 & 91.7 & 55.0 & 32.6 & 115 \\
\midrule
\multicolumn{10}{l}{\textbf{Offline}}\\
\quad MemP & \textbf{100.0} & \textbf{92.3} & 74.1 & 18.8 & 28.0 & 83.3 & 69.3{\scriptsize$\pm$2.4} & 26.3 & 5.9\,/\,83 \\
\quad ACE & 94.3 & 53.8 & 51.9 & \textbf{37.5} & 28.0 & 79.2 & 61.4{\scriptsize$\pm$5.5} & 27.9 & 7.1\,/\,129 \\
\quad EvoSkill & 97.1 & 71.8 & 54.3 & 18.8 & 37.3 & 90.3 & 65.7{\scriptsize$\pm$3.5} & 29.0 & 4.7\,/\,145 \\
\quad GEPA & \textbf{100.0} & 76.9 & 66.7 & 25.0 & 10.0 & 91.7 & 65.4{\scriptsize$\pm$4.6} & 30.1 & 6.7\,/\,115 \\
\quad TextGrad & 91.4 & 61.5 & 59.3 & 18.8 & 44.0 & 54.2 & 59.3{\scriptsize$\pm$4.5} & 32.6 & 7.0\,/\,147 \\
\quad Trace2Skill & 94.3 & 84.6 & 44.4 & 31.2 & \textbf{60.0} & 79.2 & 67.9{\scriptsize$\pm$6.4} & 28.4 & 6.4\,/\,132 \\
\midrule
\multicolumn{10}{l}{\textbf{Online}}\\
\quad MemP & 95.2 & 74.4 & 75.3 & 16.7 & 38.7 & 80.6 & 67.9{\scriptsize$\pm$2.9} & 27.4 & 85 \\
\quad ReasoningBank & 95.2 & 59.0 & 75.3 & 22.9 & 36.0 & 76.4 & 66.0{\scriptsize$\pm$2.4} & 29.0 & 93 \\
\quad ACE & 93.3 & 51.3 & \textbf{76.5} & 29.2 & 57.3 & 75.0 & 69.3{\scriptsize$\pm$6.3} & 28.3 & 113 \\
\quad A-Mem & 95.7 & 69.2 & 53.7 & 9.4 & 22.0 & 85.4 & 60.4{\scriptsize$\pm$0.4} & 30.7 & 101 \\
\quad MemRL & 97.1 & 53.8 & 66.7 & 12.5 & 20.0 & \textbf{95.8} & 63.6{\scriptsize$\pm$2.5} & 29.5 & 110 \\
\quad Dynamic Cheatsheet & 97.1 & 61.5 & 48.1 & 0.0 & 20.0 & 83.3 & 57.1{\scriptsize$\pm$4.8} & 32.2 & 126 \\
\quad TT-OPSD & 98.1 & 79.5 & 54.3 & 12.5 & 5.3 & \textbf{95.8} & 61.2{\scriptsize$\pm$1.8} & 29.3 & 99 \\
\rowcolor{AscentTint}\quad \ascentmethod{ASCENT} & \textbf{100.0} & 89.7 & \textbf{76.5} & \textbf{37.5} & 48.0 & \textbf{95.8} & \textbf{77.4}{\scriptsize$\pm$1.8} & \textbf{21.0} & 133 \\
\bottomrule
\end{tabular}%
\caption{\textbf{ALFWorld seen-stream results.} Success rate (SR, \%) by task
type and over the full stream, mean turns per episode, and runtime, given as
training GPU-hours / evaluation seconds per episode for offline methods and
seconds per episode otherwise. Bold indicates the highest SR and fewest mean
turns within each backbone.}
\label{tab:alfworld_seen}
\end{table*}

\subsection{Experimental Setting}
\label{sec:experimental_setup}

\noindent\textbf{Benchmarks.}
We evaluate on ALFWorld~\citep{shridhar2020alfworld}, a text-based household
environment, and WebShop~\citep{yao2022webshop}, a simulated web store in which
the agent searches for and buys a product matching an instruction. The ALFWorld
seen split (140 tasks across six task families) is the main stream, and the unseen
split (134 tasks in held-out rooms and layouts) tests frozen transfer and
continued OaTTT.
AppWorld~\citep{trivedi2024appworld}, which tests interactive coding across
simulated apps, is reported in Appendix~\ref{app:appworld}. More details in
Appendix~\ref{app:experimental_details}.

\noindent\textbf{OaTTT protocol.}
All online methods follow Section~\ref{sec:method_setting}
on the same task order: one pass over the stream, a single scored attempt per
task (50-turn budget on ALFWorld and WebShop), and each task scored before its
trajectory can update the method. They receive the same environment
verification signal. Offline baselines instead
build memories, skills, or prompts from the same fixed-size training split
before evaluation. They fall outside this setting and serve as reference
points.

\noindent\textbf{Baselines.}
We compare ASCENT with the base model, offline and online
self-evolution methods (Appendix~\ref{app:baselines}), and TT-OPSD, our
test-time (TT) adaptation of OPSD~\citep{zhao2026opsd}, which was designed for
static, single-turn reasoning. Its teacher receives
only the current turn's response and returned observation as privileged
information. All other configurations are the same as ASCENT.

\noindent\textbf{Evaluation metrics.}
We report
ALFWorld exact success overall and by task family, WebShop exact success
(task score $=1$) and mean score, and mean per-episode turns and wall-clock
time. On WebShop, ASCENT's adaptation gate accepts a score of at least $0.9$.
All results are reported as mean$\pm$std over 3 independent runs.

\noindent\textbf{Implementation details.}
We use Qwen3.5-4B and Qwen3.5-9B~\citep{qwen3.5}.
The base, all compared methods, and ASCENT run with the same ReAct
template~\citep{yao2022react} and greedy decoding.
More details in Appendix~\ref{app:implementation}.

\begin{table*}[t!]
\centering\footnotesize
\setlength{\tabcolsep}{2.6pt}
\renewcommand{\arraystretch}{1.08}
\begin{tabular}{@{}l cccc cccc@{}}
\toprule
& \multicolumn{4}{c}{\textbf{Qwen3.5-4B}} & \multicolumn{4}{c}{\textbf{Qwen3.5-9B}} \\
\cmidrule(lr){2-5}\cmidrule(lr){6-9}
\textbf{Method} & \textbf{Strict-SR} & Mean-Score & Turns & \makecell{Time\\{\scriptsize GPU\textperiodcentered h\,/\,s}} & \textbf{Strict-SR} & Mean-Score & Turns & \makecell{Time\\{\scriptsize GPU\textperiodcentered h\,/\,s}} \\
\midrule
Base & 17.8 & 25.9 & 39.5 & 87 & 17.2 & 27.8 & 37.5 & 166 \\
\midrule
\multicolumn{9}{l}{\textbf{Offline}}\\
\quad MemP (offline) & 33.2{\scriptsize$\pm$1.3} & 43.7{\scriptsize$\pm$2.0} & 30.0 & 6.9 / 74 & 26.6{\scriptsize$\pm$1.3} & 40.6{\scriptsize$\pm$2.8} & 31.9 & 12.3 / 148 \\
\quad ACE (offline) & 5.8{\scriptsize$\pm$0.5} & 21.4{\scriptsize$\pm$1.1} & 42.0 & 9.4 / 122 & 13.2{\scriptsize$\pm$0.6} & 22.4{\scriptsize$\pm$1.4} & 41.6 & 16.4 / 229 \\
\quad EvoSkill & 20.0{\scriptsize$\pm$1.6} & 29.3{\scriptsize$\pm$1.9} & 38.0 & 2.06 / 84 & 20.2{\scriptsize$\pm$2.1} & 33.4{\scriptsize$\pm$1.4} & 38.2 & 2.33 / 169 \\
\quad GEPA & 23.4{\scriptsize$\pm$1.2} & 37.5{\scriptsize$\pm$2.4} & 34.0 & 8.4 / 75 & 35.4{\scriptsize$\pm$3.2} & 48.6{\scriptsize$\pm$1.6} & 30.4 & 10.1 / 135 \\
\quad TextGrad & 10.4{\scriptsize$\pm$2.1} & 30.5{\scriptsize$\pm$2.5} & 35.9 & 13.3 / 79 & 18.4{\scriptsize$\pm$2.5} & 34.7{\scriptsize$\pm$2.8} & 34.5 & 14.7 / 153 \\
\quad Trace2Skill & 14.2{\scriptsize$\pm$0.6} & 21.7{\scriptsize$\pm$1.2} & 41.6 & 7.7 / 92 & 10.0{\scriptsize$\pm$4.2} & 21.0{\scriptsize$\pm$3.6} & 40.4 & 10.2 / 179 \\
\midrule
\multicolumn{9}{l}{\textbf{Online adaptation}}\\
\quad MemP & 36.8{\scriptsize$\pm$3.4} & 47.7{\scriptsize$\pm$4.8} & 28.7 & 74 & 22.4{\scriptsize$\pm$2.5} & 38.9{\scriptsize$\pm$4.1} & 32.9 & 152 \\
\quad ReasoningBank & 10.6{\scriptsize$\pm$0.7} & 21.6{\scriptsize$\pm$1.0} & 42.1 & 176 & 12.6{\scriptsize$\pm$2.1} & 21.8{\scriptsize$\pm$1.5} & 41.5 & 226 \\
\quad ACE & 10.6{\scriptsize$\pm$3.2} & 23.6{\scriptsize$\pm$4.3} & 40.6 & 107 & 16.0{\scriptsize$\pm$1.4} & 32.0{\scriptsize$\pm$2.6} & 36.1 & 144 \\
\quad A-Mem & 23.2{\scriptsize$\pm$3.1} & 36.4{\scriptsize$\pm$2.9} & 33.9 & 98 & 15.4{\scriptsize$\pm$0.6} & 32.5{\scriptsize$\pm$4.3} & 35.6 & 176 \\
\quad MemRL & 21.8{\scriptsize$\pm$1.2} & 32.3{\scriptsize$\pm$2.5} & 37.1 & 106 & 14.0{\scriptsize$\pm$4.2} & 27.4{\scriptsize$\pm$1.7} & 38.9 & 166 \\
\quad Dynamic Cheatsheet & 16.2{\scriptsize$\pm$1.5} & 28.4{\scriptsize$\pm$2.1} & 39.4 & 133 & 13.4{\scriptsize$\pm$2.9} & 27.4{\scriptsize$\pm$3.0} & 39.1 & 211 \\
\quad TT-OPSD & 25.6{\scriptsize$\pm$0.5} & 37.1{\scriptsize$\pm$1.7} & 34.6 & 100 & 24.2{\scriptsize$\pm$2.4} & 35.9{\scriptsize$\pm$2.3} & 34.8 & 108 \\
\rowcolor{AscentTint}\quad \ascentmethod{ASCENT} & \textbf{41.9}{\scriptsize$\pm$0.3} & \textbf{62.4}{\scriptsize$\pm$1.0} & \textbf{17.8} & 76 & \textbf{41.5}{\scriptsize$\pm$1.8} & \textbf{56.8}{\scriptsize$\pm$1.7} & \textbf{26.7} & 164 \\
\bottomrule
\end{tabular}
\caption{\textbf{WebShop results.} Strict-SR (\%) denotes
\texttt{task\_score}$=1$, and Mean-Score (\%) averages
\texttt{task\_score} over each 500-task stream. Offline runtime is training
GPU-hours\,/\,evaluation seconds per episode, and other entries are seconds per
episode. Bold indicates the highest Strict-SR and Mean-Score and the fewest
turns within each backbone.}
\label{tab:main_webshop}
\end{table*}

\begin{figure}[t]
\centering
\captionsetup[subfigure]{skip=1pt}
\includegraphics{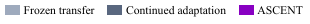}\\[1pt]
\begin{subfigure}[t]{0.49\linewidth}
\centering
\includegraphics[width=\linewidth]{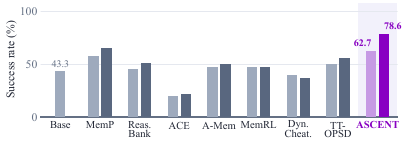}
\caption{Qwen3.5-4B}
\end{subfigure}\hfill
\begin{subfigure}[t]{0.49\linewidth}
\centering
\includegraphics[width=\linewidth]{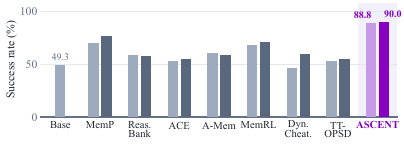}
\caption{Qwen3.5-9B}
\end{subfigure}
\caption{\textbf{Cross-scene transfer and continued OaTTT on ALFWorld unseen.}
Transfer freezes the seen-stream adaptation state, whereas continued OaTTT resumes
online updates. More in
Tables~\ref{tab:alfworld_unseen_transfer}--\ref{tab:alfworld_unseen_continual}.}
\label{fig:transfer_continual}
\end{figure}

\subsection{Main Experimental Results}
\label{sec:main_results}

\noindent\textbf{ASCENT consistently improves long-horizon task success.}
As shown in Tables~\ref{tab:alfworld_seen} and~\ref{tab:main_webshop}, ASCENT
improves exact success over the base by more than $22$ points across
ALFWorld and WebShop at both scales and exceeds the strongest online
comparator. It also achieves
the highest WebShop mean task score and outperforms every offline
in-context adaptation method on ALFWorld. ASCENT outperforms
TT-OPSD (our test-time adaptation of OPSD~\citep{zhao2026opsd}), which uses the same self-distillation but whose teacher sees only the
current turn, showing the benefit of hindsight from the complete verified
trajectory. These consistent gains show that cross-episode
parameter updates can extract reusable behavior from the same sparse
environment verification available to offline and online in-context adaptation methods.
As discussed in Sec.~\ref{sec:direct_ntp}, we omit the direct imitation and
policy-gradient baselines from these tables because they severely degrade under
OaTTT (Fig.~\ref{fig:ntp_collapse}).

\noindent\textbf{Cross-episode gains extend to difficult tasks.}
Table~\ref{tab:alfworld_seen} shows that the base rarely solves
\texttt{Heat} or \texttt{Cool}, whereas ASCENT improves both task families at
each scale. Fig.~\ref{fig:excess_wins_intro}(a) shows that ASCENT's gain over the
base is much larger on later 4B tasks, while the base remains flat.
ASCENT derives this improvement from its own verified deployment
interactions, without a stronger teacher or a pre-deployment adaptation stage.

\begin{figure}[t!]
\centering
\includegraphics[width=\textwidth]{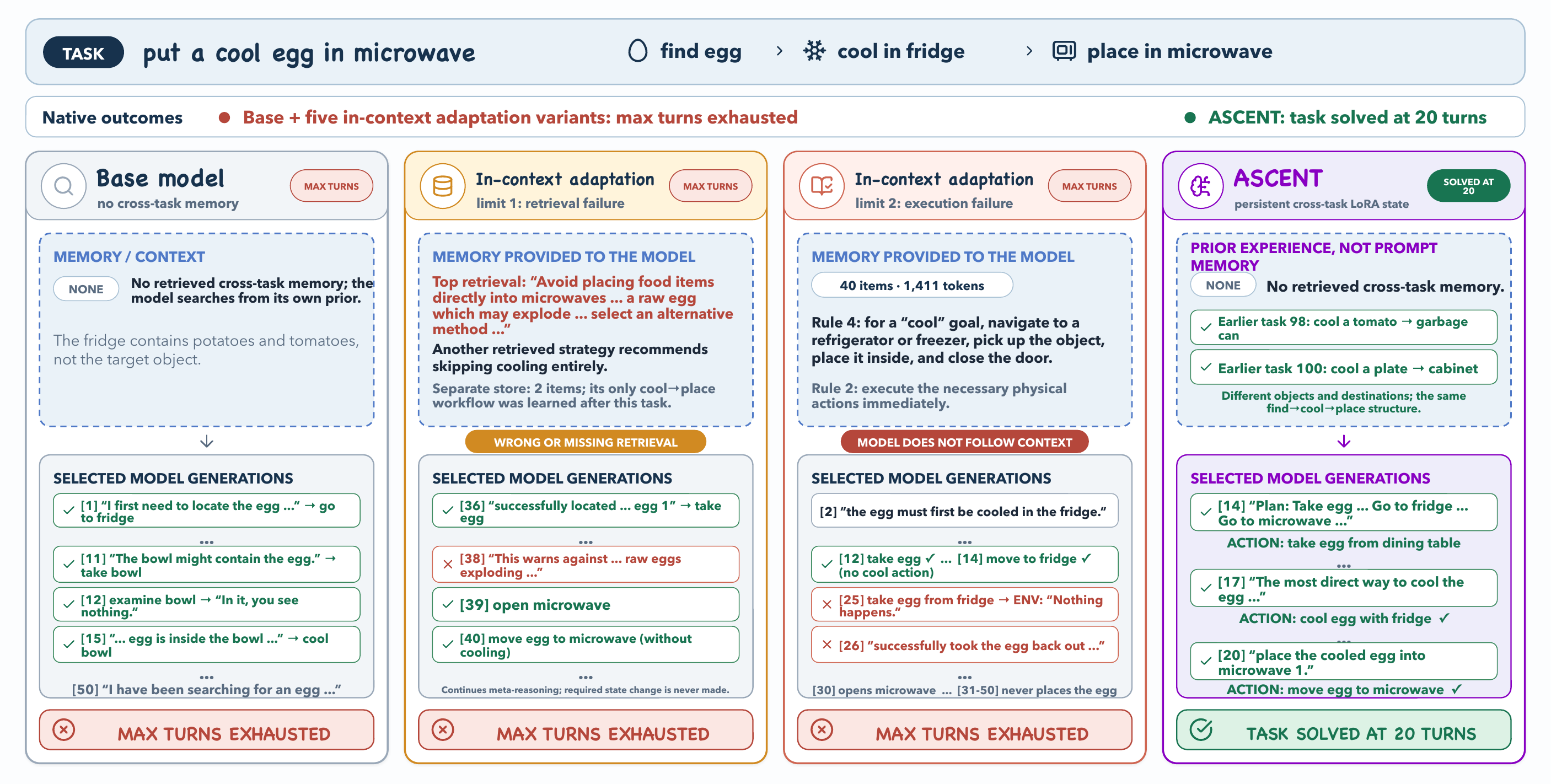}
\caption{\textbf{Quantitative analysis on Alfworld.} Columns show the base, two representative in-context
adaptation failures (irrelevant or contradictory retrieved memory, or relevant
guidance left unfollowed), and ASCENT, which carries LoRA updates from two earlier
cool-then-place tasks and retrieves no memory. Numbers are turn indices.
Checks mark valid actions, crosses mark invalid or non-action outputs,
and ellipses omit turns.}
\label{fig:g135_case}
\end{figure}

\begin{figure}[!t]
\centering
\captionsetup[subfigure]{skip=1pt}
\begin{subfigure}[t]{0.36\linewidth}
\centering
\includegraphics[width=\linewidth]{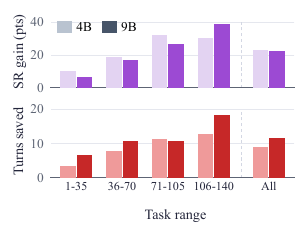}
\caption{Seen-stream diagnostics}
\end{subfigure}\hfill
\begin{subfigure}[t]{0.31\linewidth}
\centering
\includegraphics[width=\linewidth]{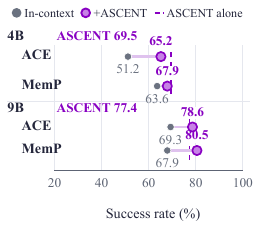}
\caption{Co-evolution}
\end{subfigure}\hfill
\begin{subfigure}[t]{0.31\linewidth}
\centering
\includegraphics[width=\linewidth]{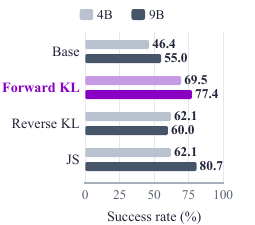}
\caption{Distillation divergence}
\end{subfigure}
\caption{\textbf{Analysis of online gains, preliminary co-evolution, and divergence ablation.} (a) Mean success-rate gain and
turns saved by ASCENT over the base, by task range and over all tasks (lighter
bars 4B, darker 9B, with turns on tasks solved by both the un-evolved base model and ASCENT in Table~\ref{tab:turn_efficiency}). (b) Preliminary in-context
and parametric co-evolution. (c) Distillation-divergence ablation,
pre-update SR (\%) at both scales, with forward KL as default.}
\label{fig:excess_wins_intro}
\end{figure}

\noindent\textbf{Fewer turns and modest online cost.}
Tables~\ref{tab:alfworld_seen} and~\ref{tab:main_webshop} show that ASCENT
reduces capped all-episode mean turns at both scales on both benchmarks. On tasks
that both the un-evolved base model and ASCENT solve, ALFWorld trajectories are
also shorter
(Sec.~\ref{sec:turn_efficiency}).
Episode time includes the attempt and LoRA update after verified successes. It
is modestly higher than the base on ALFWorld and lower on WebShop.

\noindent\textbf{The evolved policy transfers and continues learning under scene shift.}
\label{sec:unseen_results}
As shown in Fig.~\ref{fig:transfer_continual} and
Tables~\ref{tab:alfworld_unseen_transfer}--\ref{tab:alfworld_unseen_continual},
ASCENT leads the online comparators in both frozen transfer and continued
adaptation at both scales.
At 4B, continued OaTTT substantially exceeds frozen
transfer. At 9B, continuation can preserve the already strong transfer result. These results
show that seen-stream experience is retained in the evolved policy and that the
self-evolution protocol remains effective.

\subsection{Analysis and Discussion}
\label{sec:analysis_discussion}

\begin{table}[!t]
\centering
\small
\setlength{\tabcolsep}{3.5pt}
\renewcommand{\arraystretch}{1.05}
\begin{tabular}{@{}l c c c c c c c}
\toprule
& & \multicolumn{3}{c}{\textbf{Qwen3.5-4B}}
& \multicolumn{3}{c}{\textbf{Qwen3.5-9B}} \\
\cmidrule(lr){3-5}\cmidrule(l){6-8}
\textbf{Privileged information $z_i$} & \textbf{Filter}
& \textbf{SR (\%)} & \textbf{Turns} & \textbf{Valid (\%)}
& \textbf{SR (\%)} & \textbf{Turns} & \textbf{Valid (\%)} \\
\midrule
Base & -- & $46.4$ & $35.1$ & $70.1$ & $55.0$ & $32.6$ & $78.7$ \\
\midrule
& \ding{55} & $48.8$ & $32.9$ & $64.4$ & $42.9$ & $34.1$ & $66.2$ \\
\multirow{-2}{*}{Action only: $a_{i,t}$}
& \ding{51} & $48.3$ & $33.5$ & $61.2$ & $48.6$ & $32.3$ & $71.4$ \\
\cmidrule{1-8}
& \ding{55} & $62.9$ & $28.0$ & $74.6$ & $\mathbf{77.9}$ & $21.1$ & $86.6$ \\
\multirow{-2}{*}{Reasoning $+$ action $+$ observation: $(w_{i,t},o_{i,t+1})$}
& \ding{51} & $64.5$ & $28.3$ & $\mathbf{79.5}$ & $77.1$ & $23.8$ & $\mathbf{88.2}$ \\
\cmidrule{1-8}
& \ding{55} & $60.0$ & $30.6$ & $74.1$ & $75.3$ & $21.6$ & $81.8$ \\
\multirow{-2}{*}{Reasoning $+$ action: $w_{i,t}$}
& \cellcolor{AscentTint}\ding{51} & \cellcolor{AscentTint}$\mathbf{69.5}$ & \cellcolor{AscentTint}$\mathbf{26.3}$ & \cellcolor{AscentTint}$78.0$ & \cellcolor{AscentTint}$77.4$ & \cellcolor{AscentTint}$\mathbf{21.0}$ & \cellcolor{AscentTint}$85.9$ \\
\bottomrule
\end{tabular}
\caption{Privileged-information ablation on ALFWorld seen, with the
action-validity filter off (\ding{55}) or on (\ding{51}). The shaded row is ASCENT's
default (reasoning $+$ action with the filter).}
\label{tab:ablation_privileged_content}
\end{table}

\noindent\textbf{Case study on long-horizon task completion.}
\label{sec:case_study}
As shown in Fig.~\ref{fig:g135_case}, the base and five in-context
adaptation variants exhaust the turn budget on the task ``put a cool egg in
microwave,'' whereas ASCENT completes it in 20 turns. The in-context variants
fail either by retrieving irrelevant or contradictory experience or by not
following relevant guidance. ASCENT instead internalizes earlier same-family
experience in its weights and transfers it across objects and destinations.
Across whole streams, retrieving available same-type prior experience does not
reliably lead to task completion. Appendix~\ref{app:context_l1l2}
gives this retrieval analysis and the provenance of each trace.

\noindent\textbf{Interaction efficiency.}
\label{sec:turn_efficiency}
In Fig.~\ref{fig:excess_wins_intro}(a), ASCENT's turn savings over the base
grow later in the stream. They also persist on \emph{shared successes}, the
tasks that both the un-evolved base model and ASCENT solve: ASCENT uses fewer
turns than the base on $64\%$ (4B) and $67\%$ (9B) of them
(Appendix~\ref{app:turn_efficiency}, Table~\ref{tab:turn_efficiency}).

\noindent\textbf{Preliminary study on in-context and parametric co-evolution.}
\label{sec:coevo_analysis}
ASCENT stores experience in fast weights, whereas in-context adaptation stores
it in the context. In Fig.~\ref{fig:excess_wins_intro}(b), we combine ASCENT
with ACE~\citep{zhang2025ace} and MemP~\citep{fang2025memp} on ALFWorld seen. We
find that the two are complementary: adding ASCENT improves both in-context
methods at both scales. Whether retrieved memory also helps on top of ASCENT
depends on the model. On the smaller 4B model, ASCENT alone performs best, so
the extra retrieved context may distract a weaker model, in line with its
preference for shorter privileged information in
Table~\ref{tab:ablation_privileged_content}. On the relatively stronger 9B model,
both combinations exceed ASCENT alone, maybe because the model can use the
retrieved memory and its evolved weights together. Co-evolving context and
weights is thus promising for stronger agents.

\subsection{Ablation Studies}
\label{sec:ablations}

\noindent\textbf{Ablation on privileged information.}
\label{sec:privileged_ablation}
In Table~\ref{tab:ablation_privileged_content}, we study the effect of each
component of the privileged information $z_i$ (actions $a_{i,t}$,
reasoning--action responses $w_{i,t}$, and observations $o_{i,t+1}$), with and
without our proposed action-validity filter (Sec.~\ref{sec:ascent_privileged}).
We find that generally reasoning is the key component: action-only $z_i$
($a_{i,t}$) performs near or below the base, while adding reasoning ($w_{i,t}$)
consistently improves success and turn efficiency. This may be because reasoning
explains why each action was taken, such as which subgoal it serves. It also
makes up most of the tokens the student generates, while an action is only a
short command, so action-only hindsight informs few of the positions the student
is trained on.

\noindent\textbf{Smaller models benefit more from shorter, filtered privileged information.}
On the smaller 4B model, shorter
($w_{i,t}$, without observations) and carefully handled (filtered) privileged information
performs better, with higher turn efficiency (Table~\ref{tab:ablation_privileged_content}). With reasoning in $z_i$, the
filter helps most at this scale. On the relatively stronger 9B model,
providing the full trajectory ($w_{i,t}$, $o_{i,t+1}$) performs well, maybe due
to the model's stronger general language capability. Here the filter changes
success slightly while still raising the valid-action rate, possibly
because the stronger model is less affected by invalid-action turns.
Without the filter, observations may help by showing which actions the
environment flagged as invalid, so they add little once the filter removes
those turns. Our default
setting ($w_{i,t}$ with the filter) on 9B also performs comparably, with slightly
better turn efficiency.

\noindent\textbf{Ablation on distillation divergence.}
\label{sec:divergence_ablation}
As shown in Fig.~\ref{fig:excess_wins_intro}(c), ASCENT still improves on
the base at both scales when reverse KL or Jensen--Shannon (JS) divergence
replaces forward KL in \Eqref{eq:ascent_objective}. Reverse KL trails forward
KL at both scales, most at 9B. JS matches reverse KL at 4B but exceeds forward
KL at 9B. Forward KL favors coverage of the teacher distribution, whereas
reverse KL is mode-seeking~\citep{agarwal2023gkd}.

\section{Conclusion}
\label{sec:discussion}
In our Online Agentic TTT study, directly imitating or reinforcing generated tokens destabilizes
the policy. ASCENT instead self-distills verified experience into persistent LoRA
fast weights through a frozen teacher. On ALFWorld and WebShop at two model
scales, it outperforms the base and the compared baselines,
with gains that grow later in the stream and transfer to held-out scenes.

\noindent\textbf{Limitations and future work.}
ASCENT needs full next-token distributions from its frozen copy, so it applies
only to open-weight models. It gates updates on sparse episode-level
verification and weights every turn equally. Future work could approximate
finer-grained, turn-level credit. Because the teacher's privileged context contains the student's
response text, it may partially reconstruct generated tokens, which would limit the
independence of its hindsight. Isolating this effect is an open question.
The online evolution also shows potential
for environment- or user-specific agents by plugging in the matching evolved
fast weights.

\FloatBarrier
\bibliographystyle{splncs04}
\bibliography{iclr2027_conference}

\clearpage
\beginappendix
\section{Additional Method Details}
\label{app:method_details}

\subsection{What Verifier Selection and Distillation Identify}
\label{app:theory}

We separate what the verifier selects from what the distillation objective
fits. Fix a task $x$ and let $R(\tau)=\mathcal V(x,\tau)\in\{0,1\}$ be the
verification outcome that gates an update ($v_i$ in
\Eqref{eq:online_ttt_rollout}). Reported benchmark success can use a stricter
criterion. The results below concern the population of possible trajectories
under a fixed collection policy, whereas ASCENT makes a single attempt per
task in one pass over the task stream.

\subsubsection{Verifier Selection Reweights Completed Trajectories}
\label{app:theory_selection}

Let $\pi$ be the response policy that collects the trajectory, including its
decoding rule, and let $\mathbb P_\pi(\tau\mid x)$ be the resulting trajectory
distribution in the environment. If the acceptance probability
$J_x(\pi):=\Pr_{\pi,\mathcal E}(R=1\mid x)>0$, the trajectories selected for
an update follow
\begin{equation}
\mathbb P_\pi^+(\tau\mid x)
=
\frac{R(\tau)\mathbb P_\pi(\tau\mid x)}{J_x(\pi)}.
\label{eq:app_accepted_law}
\end{equation}
The selected distribution is $\mathbb P_\pi$ restricted to accepted
trajectories and renormalized. At a history $S=H_t$ just before the agent
responds, let $V^\pi(S)=\Pr_\pi(R=1\mid S)$, and let $Q^\pi(S,u)$ be the
probability of acceptance when the agent produces the complete response $u$
and then follows $\pi$. When $V^\pi(S)>0$, Bayes' rule gives
\begin{equation}
\pi^+(u\mid S)
:=
\Pr_\pi(u\mid S,R=1)
=
\pi(u\mid S)\frac{Q^\pi(S,u)}{V^\pi(S)}.
\label{eq:app_bayes_policy}
\end{equation}
Verifier selection thus reweights each response by its relative chance of
leading to acceptance~\citep{brandfonbrener2022rcsl,eysenbach2022suboptimal}.
It assigns no causal credit to individual turns of an attempt and does not
reveal the outcome of an untried response. With only one collected response at
a history, selection cannot rank alternatives there. The hindsight-conditioned
teacher supplies soft targets at the student prefixes, but whether these
targets improve later tasks must be tested empirically.

\subsubsection{Forward KL Fits Teacher Predictions at Student Prefixes}
\label{app:theory_transfer}

The student fits the teacher at the student prefixes of selected trajectories,
with the collection policy $\pi$ still fixed for task $x$. For a selected
trajectory $\tau$, let $\mathcal I(\tau)$ index the turns with nonempty
responses, which enter the student loss, and let $L_{\tau t}$ be the number of
generated tokens at turn $t$. Assume $|\mathcal I(\tau)|>0$ for selected
trajectories (otherwise condition the selected distribution on this event).
Averaging tokens within each turn and then across turns, as in
\Eqref{eq:ascent_objective}, gives the training measure
\begin{equation}
\nu_\pi(\mathrm d\tau,t,j\mid x,R=1)
=
\mathbb P_\pi^+(\mathrm d\tau\mid x)
\frac{\mathbf 1\{t\in\mathcal I(\tau)\}}{|\mathcal I(\tau)|}
\frac{\mathbf 1\{1\leq j\leq L_{\tau t}\}}{L_{\tau t}}.
\label{eq:app_training_measure}
\end{equation}
This measure draws an accepted trajectory from $\mathbb P_\pi^+$, then a turn
uniformly from $\mathcal I(\tau)$, then a token position uniformly within that
turn. Let $Z=(\tau,t,j)$ be an index drawn from it, $c_Z=c_{\tau,t,j}$ its
student prefix (defined in Sec.~\ref{sec:direct_ntp}), and $z(\tau)$ the
privileged information (defined in \Eqref{eq:privileged_record}). At this
prefix the frozen teacher supplies
$q_Z(\cdot)=p_{\theta_0}(\cdot\mid z(\tau)\mathbin{\oplus}c_Z)$, and the
population objective for a candidate adapter $\phi$ is
\begin{equation}
\mathcal L_x(\phi;\pi)
=
\mathbb E_{Z\sim\nu_\pi}
\left[
D_{\mathrm{KL}}\!\left(
q_Z\,\middle\|\,p_{\theta_0,\phi}(\cdot\mid c_Z)
\right)
\right].
\label{eq:app_population_loss}
\end{equation}
Let $\bar q(\cdot\mid c)=\mathbb E_{\nu_\pi}[q_Z(\cdot)\mid c_Z=c]$ denote the
mean teacher distribution at a student prefix $c$, which averages the teacher
distributions of the accepted trajectories through $c$, weighted as in
\Eqref{eq:ascent_objective}.

\begin{lemma}[Forward-KL population target]
\label{lem:app_forward_kl}
For any student next-token distributions at the student prefixes,
\begin{align}
\mathcal L_x(\phi;\pi)
&=
\mathbb E_{Z\sim\nu_\pi}
\left[
D_{\mathrm{KL}}\!\left(q_Z\,\middle\|\,\bar q(\cdot\mid c_Z)\right)
\right]\nonumber\\
&\quad+
\mathbb E_{Z\sim\nu_\pi}
\left[
D_{\mathrm{KL}}\!\left(
\bar q(\cdot\mid c_Z)\,\middle\|\,
p_{\theta_0,\phi}(\cdot\mid c_Z)
\right)
\right].
\label{eq:app_kl_centroid}
\end{align}
\end{lemma}

\begin{proof}
Condition on $c_Z=c$, expand the KL terms over next tokens, and use
$\bar q(\cdot\mid c)=\mathbb E[q_Z(\cdot)\mid c_Z=c]$.
The terms involving the student distribution then agree on both sides, and
averaging over student prefixes gives \Eqref{eq:app_kl_centroid}.
\end{proof}

The first term in \Eqref{eq:app_kl_centroid} does not depend on $\phi$, so an
unrestricted minimizer matches $\bar q$ at the student prefixes covered by the
training measure. The student's responses determine which prefixes appear and
how they are weighted, while the hindsight-conditioned teacher supplies their soft
targets. The identity characterizes this target only. It does not show that
the teacher's predictions help, that a finite adapter reaches the target, or
that the fitted behavior transfers to later tasks.

\subsubsection{Verification Does Not Identify Shorter Paths}
\label{app:theory_detours}

An accepted trajectory can contain valid but unnecessary turns. Suppose
\emph{finish} succeeds now, while \emph{wait} leads to success on the next
turn. Both responses can have $Q^\pi(S,u)=1$, so verifier selection prefers
neither, and the action-validity filter keeps both when the environment executes
them.

Let $T$ be the episode length under $\pi$. Selecting accepted episodes shifts
its mean by
\begin{equation}
\mathbb E_\pi[T\mid R=1]-\mathbb E_\pi[T]
=
\frac{\operatorname{Cov}_\pi(T,R)}{J_x(\pi)}.
\label{eq:app_length_selection}
\end{equation}
The sign of this shift follows the covariance between length and acceptance,
so verification alone does not imply shorter accepted trajectories. With a
turn cap, failed episodes may run to the limit, so a drop in capped
all-episode mean turns can reflect more successes even when successful paths
are no shorter. Appendix~\ref{app:turn_efficiency} also compares turns on
shared successes, the tasks that both the un-evolved base model and ASCENT
solve. This comparison describes
turn differences on those tasks. It does not identify which individual turns
the updated policy learned to avoid.

\section{Additional Experimental Details}
\label{app:experimental_details}

\subsection{Evaluation Protocol and Comparison Controls}

All methods in a paired comparison run on the same task order. After each scored
episode, every online method receives the same environment verification
signal and may use it to update its memory, skills, or model weights. All
in-context adaptation methods start from the ReAct template that ASCENT uses
to collect its attempts. Offline self-evolution methods instead build their
adaptation state from the same fixed-size training split before evaluation.
These controls match the task stream, prompt template, and offline data
volume, but leave each method free to learn in its own way.
On WebShop, exact success requires a task score of $1$, whereas ASCENT's
adaptation gate accepts a score of at least $0.9$.

\subsection{Baseline Details}
\label{app:baselines}

\noindent\textbf{Base.}
The base is the initial LLM with the shared ReAct interaction
template~\citep{yao2022react} and makes no cross-task update.

\noindent\textbf{Offline in-context adaptation.}
MemP~\citep{fang2025memp}, ACE~\citep{zhang2025ace},
EvoSkill~\citep{alzubi2026evoskill}, GEPA~\citep{agrawal2025gepa},
TextGrad~\citep{yuksekgonul2024textgrad}, and
Trace2Skill~\citep{ni2026trace2skill} build memories, skills, or prompts from
the fixed training split before evaluation.

\noindent\textbf{Online in-context adaptation.}
MemP~\citep{fang2025memp}, ReasoningBank~\citep{ouyang2025reasoningbank},
ACE~\citep{zhang2025ace}, A-Mem~\citep{xu2025amem},
MemRL~\citep{zhang2026memrl}, and Dynamic
Cheatsheet~\citep{suzgun2025dynamiccheatsheet} update memory or skills after a
scored episode, so the new context can first affect the next task.

\begin{figure}[t]
\centering
\includegraphics[width=\linewidth]{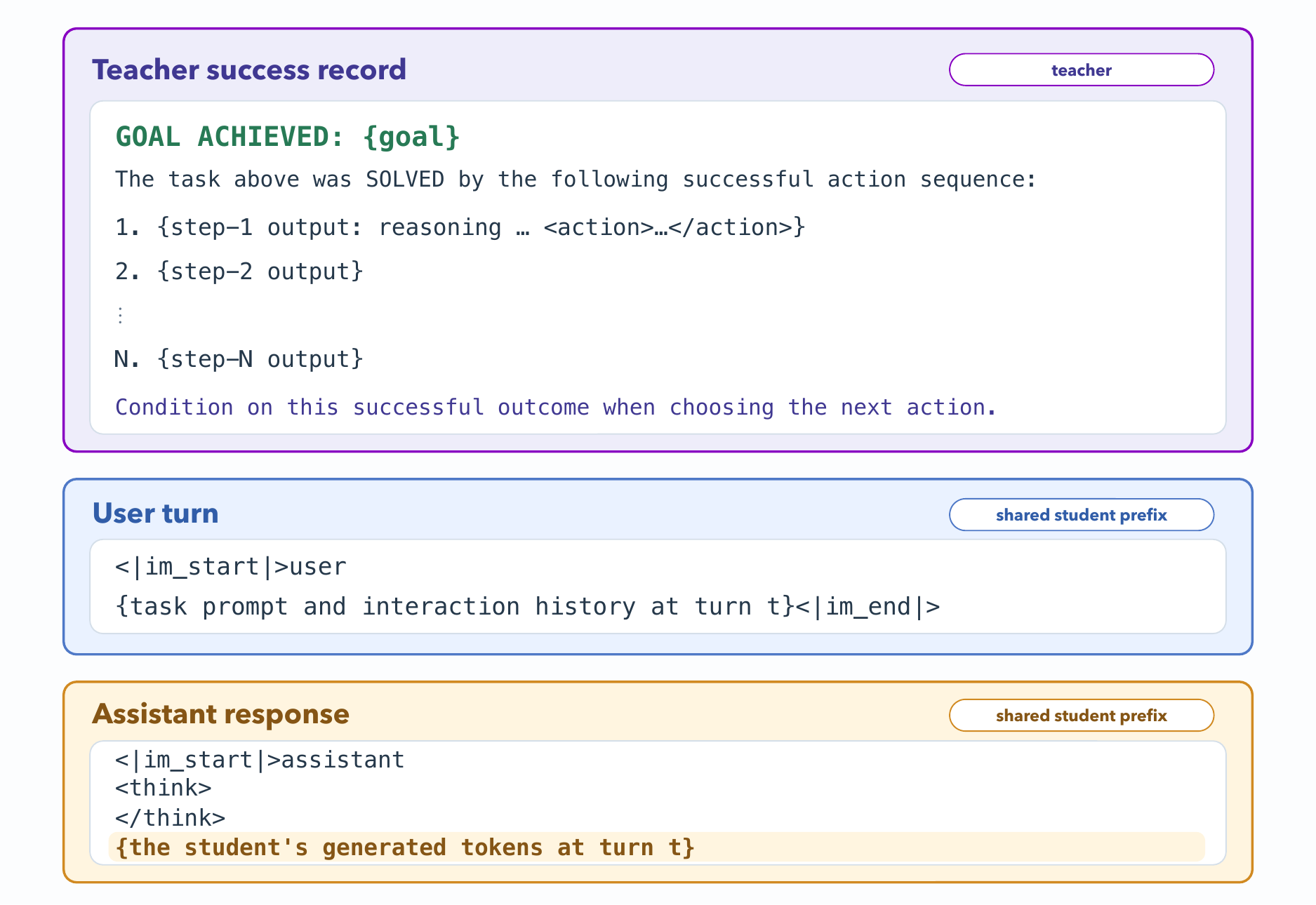}
\caption{\textbf{Teacher prompt template.} For a verifier-accepted episode, the
privileged information $z_i$ is prepended to each shared student prefix. The
empty think block reflects disabled thinking mode. Each response's reasoning is
the visible ReAct rationale before the action tag.}
\label{fig:teacher_prompt_template}
\end{figure}

\noindent\textbf{Online parametric adaptation.}
TT-OPSD is our test-time (TT) adaptation of OPSD~\citep{zhao2026opsd} to the
one-pass task stream. Its teacher receives the turn-local privileged
information $z^{\mathrm{TT}}_{i,t}=(x_i,w_{i,t},o_{i,t+1})$ together with each
student prefix $c_{i,t,j}$, so the teacher's extra information covers only the
current response and returned observation, without the completed trajectory. Its
verifier gate, student prefixes, forward-KL objective, and optimizer match
ASCENT, but it keeps invalid-action turns in accepted episodes.
Sec.~\ref{sec:direct_ntp} also studies direct imitation (Online
ungated imitation, Online RFT~\citep{yuan2023rft,shao2024deepseekmath}, and Online
RFT+KL) and single-attempt policy gradient (Online
REINFORCE~\citep{williams1992reinforce} and Online
REINFORCE++~\citep{hu2025reinforcepp}).

\noindent\textbf{Single-attempt policy gradient.}
Online REINFORCE and Online REINFORCE++ use ASCENT's LoRA configuration and
optimizer and update after each episode. Online REINFORCE++ uses a per-token KL penalty to the frozen
base with $\beta=0.01$ and ratio clipping with $\epsilon=0.2$. Where REINFORCE++ uses batch statistics, Online REINFORCE++ normalizes each token
advantage by the mean and standard deviation of all token advantages in the
stream so far, including the current episode.

\subsection{Implementation Details}
\label{app:implementation}

\noindent\textbf{Turn budget and action validity.}
Every ALFWorld and WebShop attempt, in evaluation and in online training, has a
50-turn budget. An attempt still unfinished at turn 50 stops there and is
scored as a failure. Each environment decides validity from its own response
to the action. ALFWorld executes an action only if it belongs to its current
available-action list and otherwise returns ``Nothing happens.'' We set
$d_{i,t}=1$ for listed actions and $d_{i,t}=0$ for unparsed or unavailable
ones, so a listed action that has no effect still counts as valid. WebShop
executes \texttt{search[$\cdot$]} only on a page with a search bar and
\texttt{click[$\cdot$]} only on an element of the current page. We set
$d_{i,t}=0$ for other or unparsed actions. In AppWorld, $d_{i,t}=0$ when the
turn contains no parseable code or its execution raises an error.

\begin{table*}[t!]
\centering\small
\setlength{\tabcolsep}{4.5pt}
\renewcommand{\arraystretch}{1.12}
\begin{tabular}{l cccccc c cc}
\toprule
\textbf{Method} & Pick & Look & Clean & Heat & Cool & Pick2 & \textbf{Avg.SR} & \makecell{Mean turns\\{\scriptsize(all)}} & \makecell{Time\\{\scriptsize s/ep}} \\
\midrule
\rowcolor{gray!22}\multicolumn{10}{l}{\textit{Qwen3.5-4B}}\\
\midrule
Base & 87.5 & 38.9 & 29.0 & 8.7 & 28.6 & 76.5 & 43.3 & 36.2 & 107 \\
\midrule
\multicolumn{10}{l}{\textbf{Offline}}\\
\quad MemP (offline) & 91.7 & 72.2 & 64.5 & 8.7 & 52.4 & 76.5 & 60.4{\scriptsize$\pm$4.5} & 30.0 & 69 \\
\quad ACE (offline) & 79.2 & 33.3 & 41.9 & 8.7 & 42.9 & 23.5 & 39.6{\scriptsize$\pm$3.2} & 38.5 & 106 \\
\quad EvoSkill & 83.3 & 40.7 & 29.0 & 5.8 & 47.6 & 58.8 & 43.0{\scriptsize$\pm$4.6} & 37.4 & 146 \\
\quad GEPA & 95.8 & 38.9 & 38.7 & \textbf{47.8} & \textbf{76.2} & 70.6 & 60.4{\scriptsize$\pm$4.9} & 32.8 & 95 \\
\quad TextGrad & 79.2 & 66.7 & 22.6 & 8.7 & 19.0 & 47.1 & 38.8{\scriptsize$\pm$6.4} & 38.6 & 112 \\
\quad Trace2Skill & 79.2 & 50.0 & 16.1 & 4.3 & 19.0 & 76.5 & 38.1{\scriptsize$\pm$5.3} & 39.4 & 127 \\
\midrule
\multicolumn{10}{l}{\textbf{Online-built artifacts (frozen on unseen)}}\\
\quad MemP & 87.5 & \textbf{88.9} & \textbf{80.6} & 4.3 & 42.9 & 29.4 & 57.5{\scriptsize$\pm$6.9} & 31.1 & 66 \\
\quad ReasoningBank & 87.5 & 16.7 & 41.9 & 21.7 & 52.4 & 47.1 & 45.5{\scriptsize$\pm$7.6} & 36.0 & 91 \\
\quad ACE & 70.8 & 11.1 & 3.2 & 0.0 & 4.8 & 35.3 & 20.1{\scriptsize$\pm$7.4} & 45.2 & 134 \\
\quad A-Mem & 91.7 & 38.9 & 41.9 & 8.7 & 42.9 & 58.8 & 47.0{\scriptsize$\pm$9.0} & 34.5 & 84 \\
\quad MemRL & 79.2 & 38.9 & 45.2 & 17.4 & 33.3 & 70.6 & 47.0{\scriptsize$\pm$6.1} & 35.6 & 88 \\
\quad Dynamic Cheatsheet & 79.2 & 27.8 & 12.9 & 13.0 & 38.1 & \textbf{82.4} & 39.6{\scriptsize$\pm$7.9} & 38.7 & 102 \\
\quad TT-OPSD & \textbf{98.6} & 66.7 & 35.5 & 5.8 & 33.3 & 70.6 & 50.0{\scriptsize$\pm$4.8} & 33.9 & 67 \\
\rowcolor{AscentTint}\quad \ascentmethod{ASCENT} & 86.1 & 46.3 & 72.0 & 20.3 & 69.8 & 78.4 & \textbf{62.7}{\scriptsize$\pm$7.5} & \textbf{28.4} & 94 \\
\midrule
\rowcolor{gray!22}\multicolumn{10}{l}{\textit{Qwen3.5-9B}}\\
\midrule
Base & 91.7 & 88.9 & 38.7 & 4.3 & 0.0 & 88.2 & 49.3 & 34.3 & 138 \\
\midrule
\multicolumn{10}{l}{\textbf{Offline}}\\
\quad MemP (offline) & \textbf{100.0} & \textbf{100.0} & 71.0 & 21.7 & 33.3 & 82.4 & 67.2{\scriptsize$\pm$6.7} & 27.8 & 94 \\
\quad ACE (offline) & 83.3 & 66.7 & 41.9 & 39.1 & 42.9 & 88.2 & 58.2{\scriptsize$\pm$7.3} & 30.9 & 141 \\
\quad EvoSkill & 95.8 & 70.4 & 46.2 & 23.2 & 38.1 & 90.2 & 58.7{\scriptsize$\pm$4.3} & 31.9 & 163 \\
\quad GEPA & \textbf{100.0} & 83.3 & 58.1 & 23.9 & 11.9 & 82.4 & 59.0{\scriptsize$\pm$9.0} & 30.8 & 139 \\
\quad TextGrad & \textbf{100.0} & 77.8 & 64.5 & 17.4 & 85.7 & 64.7 & 67.9{\scriptsize$\pm$4.7} & 30.3 & 173 \\
\quad Trace2Skill & \textbf{100.0} & 94.4 & 41.9 & 34.8 & 57.1 & 70.6 & 64.2{\scriptsize$\pm$5.1} & 32.4 & 157 \\
\midrule
\multicolumn{10}{l}{\textbf{Online-built artifacts (frozen on unseen)}}\\
\quad MemP & \textbf{100.0} & 83.3 & 71.0 & 13.0 & 76.2 & 82.4 & 70.1{\scriptsize$\pm$4.3} & 26.3 & 84 \\
\quad ReasoningBank & 95.8 & 66.7 & 67.7 & 26.1 & 28.6 & 64.7 & 59.0{\scriptsize$\pm$8.1} & 32.8 & 120 \\
\quad ACE & 83.3 & 66.7 & 19.4 & 17.4 & 71.4 & 82.4 & 53.0{\scriptsize$\pm$5.9} & 35.1 & 151 \\
\quad A-Mem & 95.8 & 72.2 & 45.2 & 26.1 & 47.6 & 88.2 & 60.4{\scriptsize$\pm$7.1} & 32.0 & 111 \\
\quad MemRL & \textbf{100.0} & 88.9 & 54.8 & 30.4 & 57.1 & 88.2 & 67.9{\scriptsize$\pm$6.6} & 29.5 & 107 \\
\quad Dynamic Cheatsheet & 95.8 & 61.1 & 19.4 & 21.7 & 9.5 & 88.2 & 46.3{\scriptsize$\pm$5.2} & 35.4 & 151 \\
\quad TT-OPSD & \textbf{100.0} & 72.2 & 48.4 & 8.7 & 4.8 & \textbf{96.1} & 53.2{\scriptsize$\pm$5.0} & 31.8 & 89 \\
\rowcolor{AscentTint}\quad \ascentmethod{ASCENT} & 97.2 & 81.5 & \textbf{90.3} & \textbf{72.5} & \textbf{96.8} & 94.1 & \textbf{88.8}{\scriptsize$\pm$6.6} & \textbf{17.0} & 74 \\
\bottomrule
\end{tabular}%
\caption{\textbf{ALFWorld held-out transfer.} Success rate (SR, \%) by task
type and over the full held-out stream, mean turns per episode, and runtime
(s/episode). Seen-stream artifacts remain fixed. Bold indicates the highest
SR (including ties) and fewest mean turns within each backbone.
\label{tab:alfworld_unseen_transfer}}
\end{table*}

\begin{table*}[t!]
\centering\small
\setlength{\tabcolsep}{4.5pt}
\renewcommand{\arraystretch}{1.12}
\begin{tabular}{l cccccc c cc}
\toprule
\textbf{Method} & Pick & Look & Clean & Heat & Cool & Pick2 & \makecell{\textbf{Avg.SR}\\{\scriptsize(\,$\Delta$ vs Transfer)}} & \makecell{Mean turns\\{\scriptsize(all)}} & \makecell{Time\\{\scriptsize s/ep}} \\
\midrule
\rowcolor{gray!22}\multicolumn{10}{l}{\textit{Qwen3.5-4B}}\\
\midrule
Base & 87.5 & 38.9 & 29.0 & 8.7 & 28.6 & 76.5 & 43.3 & 36.2 & 107 \\
\midrule
\multicolumn{10}{l}{\textbf{Online adaptation on unseen stream}}\\
\quad MemP & \textbf{100.0} & \textbf{94.4} & 74.2 & 21.7 & 33.3 & 64.7 & 64.9{\scriptsize$\pm$4.8}~{\scriptsize($+$7.4)} & 28.2 & 58 \\
\quad ReasoningBank & 87.5 & 33.3 & 48.4 & 26.1 & 52.4 & 52.9 & 50.7{\scriptsize$\pm$7.3}~{\scriptsize($+$5.2)} & 35.0 & 89 \\
\quad ACE & 75.0 & 11.1 & 3.2 & 0.0 & 9.5 & 35.3 & 21.6{\scriptsize$\pm$5.2}~{\scriptsize($+$1.5)} & 45.5 & 225 \\
\quad A-Mem & 87.5 & 55.6 & 51.6 & 21.7 & 28.6 & 52.9 & 50.0{\scriptsize$\pm$3.9}~{\scriptsize($+$3.0)} & 33.5 & 85 \\
\quad MemRL & 79.2 & 38.9 & 51.6 & 21.7 & 42.9 & 47.1 & 47.8{\scriptsize$\pm$8.7}~{\scriptsize($+$0.8)} & 34.3 & 86 \\
\quad Dynamic Cheatsheet & 79.2 & 61.1 & 29.0 & 0.0 & 23.8 & 35.3 & 37.3{\scriptsize$\pm$6.1}~{\scriptsize($-$2.3)} & 37.5 & 99 \\
\quad TT-OPSD & 97.2 & 81.5 & 39.8 & 15.9 & 28.6 & \textbf{90.2} & 56.2{\scriptsize$\pm$3.1}~{\scriptsize($+$6.2)} & 31.5 & 90 \\
\rowcolor{AscentTint}\quad \ascentmethod{ASCENT} & 90.3 & 83.3 & \textbf{83.9} & \textbf{53.6} & \textbf{77.8} & 82.4 & \textbf{78.6}{\scriptsize$\pm$5.1}~{\scriptsize($+$15.9)} & \textbf{21.1} & 98 \\
\midrule
\rowcolor{gray!22}\multicolumn{10}{l}{\textit{Qwen3.5-9B}}\\
\midrule
Base & 91.7 & 88.9 & 38.7 & 4.3 & 0.0 & 88.2 & 49.3 & 34.3 & 138 \\
\midrule
\multicolumn{10}{l}{\textbf{Online adaptation on unseen stream}}\\
\quad MemP & \textbf{100.0} & \textbf{100.0} & 83.9 & 13.0 & 76.2 & 94.1 & 76.9{\scriptsize$\pm$5.2}~{\scriptsize($+$6.8)} & 24.9 & 74 \\
\quad ReasoningBank & 95.8 & 66.7 & 58.1 & 17.4 & 38.1 & 70.6 & 57.5{\scriptsize$\pm$4.8}~{\scriptsize($-$1.5)} & 33.1 & 124 \\
\quad ACE & 91.7 & 55.6 & 38.7 & 13.0 & 66.7 & 76.5 & 55.2{\scriptsize$\pm$7.1}~{\scriptsize($+$2.2)} & 33.2 & 142 \\
\quad A-Mem & 83.3 & 72.2 & 64.5 & 21.7 & 33.3 & 82.4 & 59.0{\scriptsize$\pm$3.7}~{\scriptsize($-$1.4)} & 31.7 & 116 \\
\quad MemRL & 95.8 & 94.4 & 74.2 & 21.7 & 52.4 & 94.1 & 70.9{\scriptsize$\pm$7.2}~{\scriptsize($+$3.0)} & 27.0 & 99 \\
\quad Dynamic Cheatsheet & 95.8 & 61.1 & 54.8 & 17.4 & 47.6 & 88.2 & 59.7{\scriptsize$\pm$8.6}~{\scriptsize($+$13.4)} & 29.1 & 121 \\
\quad TT-OPSD & 97.2 & 85.2 & 46.2 & 13.0 & 4.8 & \textbf{96.1} & 54.7{\scriptsize$\pm$0.9}~{\scriptsize($+$1.5)} & 30.3 & 111 \\
\rowcolor{AscentTint}\quad \ascentmethod{ASCENT} & 98.6 & 92.6 & \textbf{88.2} & \textbf{72.5} & \textbf{95.2} & \textbf{96.1} & \textbf{90.0}{\scriptsize$\pm$3.6}~{\scriptsize($+$1.2)} & \textbf{15.7} & 107 \\
\bottomrule
\end{tabular}%
\caption{\textbf{ALFWorld continued OaTTT.} Pre-update success rate
(SR, \%) by task type and over the held-out stream, mean turns per episode,
and runtime (s/episode). $\Delta$ is the change from frozen transfer in
Table~\ref{tab:alfworld_unseen_transfer}. Bold indicates the highest SR
(including ties) and fewest mean turns within each backbone.
\label{tab:alfworld_unseen_continual}}
\end{table*}

\noindent\textbf{Teacher input.}
The shared student prefix $c_{i,t,j}$ includes any retrieved context used
during the attempt. The student reads $c_{i,t,j}$, whereas the teacher
reads $z_i\oplus c_{i,t,j}$, with teacher token positions shifted by $|z_i|$.
The trajectory text in $z_i$ conditions the teacher. It is not itself a
prediction target. Fig.~\ref{fig:teacher_prompt_template} shows this
arrangement for reasoning--action responses.

\noindent\textbf{Optimization.}
The teacher is a frozen copy of the initial LLM, $p_{\theta_0}$, so teacher
inference simply disables the student's LoRA adapter and returns detached
full-vocabulary logits. LoRA adapts the \texttt{down\_proj} modules with rank
16 and scale 32. After each verifier-accepted episode, ASCENT takes two
optimizer steps on the forward-KL loss of \Eqref{eq:ascent_objective} with
AdamW~\citep{loshchilov2019adamw} and learning rate $1.5\times10^{-4}$. The
adapter and optimizer state carry across tasks, while the trajectory is discarded
after its update. All experiments run on a single NVIDIA H100 GPU.

\subsection{Full Benchmark Results}
\label{app:full_results}

Tables~\ref{tab:alfworld_unseen_transfer}--\ref{tab:alfworld_unseen_continual}
report the full ALFWorld unseen results for frozen transfer and continued OaTTT,
which Fig.~\ref{fig:transfer_continual} summarizes in the main paper. The full
ALFWorld seen and WebShop results, including offline methods, are in the main
paper (Tables~\ref{tab:alfworld_seen} and~\ref{tab:main_webshop}).

\subsection{Additional Evaluation on AppWorld}
\label{app:appworld}

\begin{table}[!tb]
\centering\small
\setlength{\tabcolsep}{9pt}
\renewcommand{\arraystretch}{1.08}
\begin{tabular}{@{}l cccc@{}}
\toprule
& \multicolumn{2}{c}{\textbf{Test-Normal}}
& \multicolumn{2}{c}{\textbf{Test-Challenge}} \\
\cmidrule(lr){2-3}\cmidrule(l){4-5}
\textbf{Method} & \textbf{TGC} & \textbf{SGC}
& \textbf{TGC} & \textbf{SGC} \\
\midrule
Base & 20.2 & 5.4 & 15.8 & \textbf{5.8} \\
\midrule
MemP & 22.0{\scriptsize$\pm$2.1} & 1.8{\scriptsize$\pm$0.2} & 16.1{\scriptsize$\pm$3.3} & \textbf{5.8}{\scriptsize$\pm$}0.7 \\
ReasoningBank & 19.6{\scriptsize$\pm$1.9} & 3.6{\scriptsize$\pm$0.7} & 19.2{\scriptsize$\pm$2.8} & 4.3{\scriptsize$\pm$1.1} \\
ACE & 22.0{\scriptsize$\pm$4.1} & 3.6{\scriptsize$\pm$1.2} & 14.4{\scriptsize$\pm$3.9} & 2.2{\scriptsize$\pm$2.0} \\
A-Mem & 14.3{\scriptsize$\pm$0.4} & 8.9{\scriptsize$\pm$1.2} & 11.3{\scriptsize$\pm$1.1} & 2.2{\scriptsize$\pm$0.4} \\
Dynamic Cheatsheet & 18.5{\scriptsize$\pm$2.3} & 1.8{\scriptsize$\pm$0.6} & 14.9{\scriptsize$\pm$1.4} & 2.9{\scriptsize$\pm$0.5} \\
TT-OPSD & 19.1{\scriptsize$\pm$3.2} & 5.7{\scriptsize$\pm$0.8} & 16.2{\scriptsize$\pm$2.0} & 3.8{\scriptsize$\pm$1.2} \\
\rowcolor{AscentTint}\ascentmethod{ASCENT}
& \textbf{26.2}{\scriptsize$\pm$1.3} & \textbf{9.5}{\scriptsize$\pm$0.9} & \textbf{20.0}{\scriptsize$\pm$2.5} & 5.3{\scriptsize$\pm$0.9} \\
\bottomrule
\end{tabular}
\caption{\textbf{AppWorld online results.} Task-goal completion
(TGC, \%) and scenario-goal completion (SGC, \%) on Test-Normal and
Test-Challenge with Qwen3.5-4B.}
\label{tab:appworld}
\end{table}

AppWorld~\citep{trivedi2024appworld} tests interactive code generation across
simulated applications. We evaluate Qwen3.5-4B on Test-Normal (168 tasks in 56
scenarios) and Test-Challenge (417 tasks in 139 scenarios). Each task has one
scored attempt with a 30-interaction budget, and online methods receive the
episode's verification outcome only after that attempt. Task Goal Completion
(TGC) counts tasks that pass every state-based test, and Scenario Goal
Completion (SGC) counts scenarios in which every task variant succeeds.
Table~\ref{tab:appworld} reports these official metrics under the OaTTT
protocol.

\noindent\textbf{Task-level gains extend to interactive coding.}
As shown in Table~\ref{tab:appworld}, ASCENT improves TGC over the base on both
splits and leads every online comparator in the table on Test-Normal TGC and
SGC. On Test-Challenge, its task-level gain does not translate into a
scenario-level lead.

\section{Additional Analyses and Results}
\label{app:additional_results}

\subsection{Teacher Evolution in OaTTT}
\label{app:teacher_dynamics}

ASCENT keeps the teacher fixed while the student changes. To analyze this
choice, we study teacher evolution in OaTTT: the teacher gets its own adapter
$\bar\phi_i$, which follows the student.
Let $Q_{\bar\phi}(\cdot\mid z,c)=p_{\theta_0,\bar\phi}(\cdot\mid z\mathbin{\oplus}c)$
denote this teacher at a student prefix $c$ with privileged information $z$,
and let $Q_{\bar\phi_i}(\tau_i)$ collect its targets at the student prefixes of
an accepted episode. With token indices suppressed, episode $i$ gives
\begin{equation}
\tau_i\sim P_{\phi_i},\qquad
q_i=Q_{\bar\phi_i}(\tau_i),\qquad
\phi_{i+1}=\mathsf{Evolve}\bigl(\phi_i;\mathcal L_i^{\mathrm{ASCENT}}[q_i]\bigr).
\label{eq:teacher_tracking_system}
\end{equation}
Here $P_{\phi_i}$ is the trajectory distribution of the student
$p_{\theta_0,\phi_i}$ on task $x_i$, $\mathsf{Evolve}$ is the operator in
\Eqref{eq:ascent_update}, and $\mathcal L_i^{\mathrm{ASCENT}}[q_i]$ is
\Eqref{eq:ascent_objective} with targets $q_i$. ASCENT sets $\bar\phi_i=0$, so
its teacher is always $p_{\theta_0}$.

\noindent\textbf{EMA and snapshot teachers.}
The evolving teachers keep the hindsight context, loss, and optimizer of ASCENT
and differ from it only in letting $\bar\phi$ follow the student. An
exponential moving average (EMA) teacher applies
$\bar\phi\leftarrow\beta\bar\phi+(1-\beta)\phi$ after every optimizer step. A
periodic snapshot teacher starts at $\bar\phi=0$ and copies the student's
adapter after every $N$-th update (one update per accepted episode), like the
target network of deep Q-learning~\citep{mnih2015human}. A fixed teacher still
gives different targets across episodes, because $z_i$ and the student prefixes
change. An evolving teacher also changes the rule that turns this evidence into
supervision. We run these variants with Qwen3.5-4B on the ALFWorld seen stream
(Table~\ref{tab:teacher_target_dynamics}).

\begin{table}[t]
\centering\small
\setlength{\tabcolsep}{4.5pt}
\renewcommand{\arraystretch}{1.08}
\begin{tabular}{l c c c}
\toprule
\textbf{Teacher} & Schedule & \textbf{SR (\%)} &
\makecell{Early$\rightarrow$late\\SR (\%)} \\
\midrule
\multicolumn{4}{l}{\textbf{Fixed}}\\
\rowcolor{AscentTint}\quad \ascentmethod{ASCENT} (frozen $\pi_0$) & fixed & \textbf{69.5} & 59.0$\rightarrow$75.2 \\
\midrule
\multicolumn{4}{l}{\textbf{EMA of the student}}\\
\quad EMA & $\beta=0.999$ & 61.4 & 48.6$\rightarrow$77.1 \\
\quad EMA & $\beta=0.99$ & 57.1 & 54.3$\rightarrow$57.1 \\
\quad EMA & $\beta=0.90$ & 24.3 & 48.6$\rightarrow$2.9 \\
\midrule
\multicolumn{4}{l}{\textbf{Periodic student snapshot}}\\
\quad Snapshot & $N=30$ & 55.7 & 51.4$\rightarrow$68.6 \\
\quad Snapshot & $N=20$ & 63.6 & 57.1$\rightarrow$62.9 \\
\quad Snapshot & $N=10$ & 45.7 & 60.0$\rightarrow$25.7 \\
\bottomrule
\end{tabular}
\caption{\textbf{Analysis of teacher evolution in OaTTT on ALFWorld seen.} The snapshot
teacher is a copy of the student that is refreshed every $N$ updates (the
frozen $\pi_0$ before the first refresh). Early and late refer to tasks 1--35
and 106--140, respectively.}
\label{tab:teacher_target_dynamics}
\end{table}

\begin{figure}[t]
\centering
\captionsetup[subfigure]{skip=1pt}
\includegraphics[width=\linewidth]{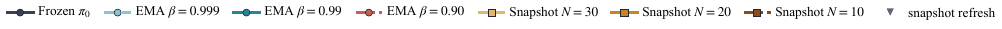}\\[1pt]
\begin{subfigure}[t]{0.325\linewidth}
\centering
\includegraphics[width=\linewidth]{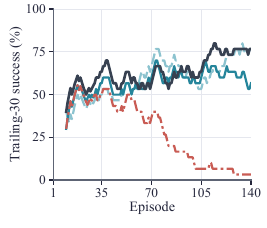}
\caption{EMA teachers}
\end{subfigure}\hfill
\begin{subfigure}[t]{0.325\linewidth}
\centering
\includegraphics[width=\linewidth]{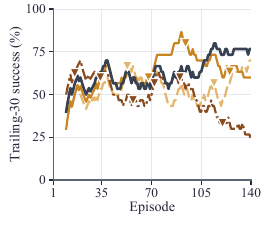}
\caption{Periodic snapshots}
\end{subfigure}\hfill
\begin{subfigure}[t]{0.325\linewidth}
\centering
\includegraphics[width=\linewidth]{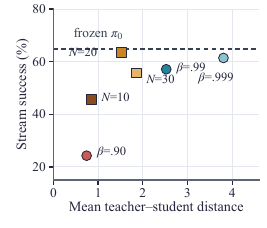}
\caption{Success vs.\ teacher distance}
\end{subfigure}
\caption{\textbf{Dynamics of teacher evolution in OaTTT on ALFWorld seen.} Qwen3.5-4B with
evolving teachers: (a,b) trailing 30-episode success against the frozen teacher
in the same setup, with triangles at snapshot refreshes, and (c) stream success
against the mean teacher--student adapter distance $\|\bar\phi-\phi\|_2$ before
each update, with the frozen teacher as the dashed line.}
\label{fig:teacher_evo}
\end{figure}

\noindent\textbf{Faster teacher evolution strengthens the feedback loop.}
An evolving teacher returns student drift as supervision on later tasks.
Stop-gradient blocks this path only within an update, leaving it open across tasks. As shown
in Table~\ref{tab:teacher_target_dynamics}, the fastest-evolving teachers (EMA
with $\beta=0.90$ and snapshots with $N=10$) lose success along the stream,
whereas slower ones hold or improve. The slowest EMA ($\beta=0.999$) ends the
stream level with the frozen teacher in the same setup
(Fig.~\ref{fig:teacher_evo}(a)). No evolving teacher matches that teacher
over the whole stream. Teacher evolution thus need not lose success along the stream, and here it
brings no net gain. These results order the schedules without locating the
evolution rate at which success starts to fall, because $\beta$ or $N$ alone does not set the lag in
episodes. An EMA remembers about $(1-\beta)^{-1}$ optimizer steps, about half as
many accepted updates, and because the verifier gates updates, the lag of
either teacher in episodes also depends on the success rate.

\noindent\textbf{Success falls when the teacher stays close to the student.}
Fig.~\ref{fig:teacher_evo}(c) relates stream success to how far the teacher's
adapter stays from the student's. The two teachers closest to the student, EMA
with $\beta=0.90$ and snapshots refreshed every 10 updates, succeed least. Once
the mean distance exceeds about $1.5$, success levels off between $56\%$ and
$64\%$, below the $65\%$ of the frozen teacher in the same setup. The snapshots
show this loop over time (Fig.~\ref{fig:teacher_evo}(b)): each refresh hands the
teacher the student's latest weights, and with frequent refreshes success
declines through the second half of the stream.

\noindent\textbf{Gating and filtering do not remove the second feedback path of an evolving teacher.}
The verification outcome selects episodes and assigns no credit to individual
turns. The action-validity filter removes invalid-action turns from $z_i$
and cannot judge the remaining valid turns. With an evolving teacher, those
turns also shape the model that supplies later targets, a second feedback path
that neither gating nor filtering removes.

\noindent\textbf{Frozen and evolving teachers define different targets.}
Let $\mu_{\phi}^{+}(z\mid c)$ be the distribution of privileged information
from verifier-accepted trajectories of the student with adapter $\phi$ at a
student prefix $c$. Under a fixed collection distribution
(Appendix~\ref{app:theory_transfer}), an unrestricted student fitted by forward
KL to the frozen teacher $Q_0$ ($\bar\phi=0$) has the centroid solution
\begin{equation}
p^{*}(\cdot\mid c)
=\mathbb E_{z\sim\mu_{\phi}^{+}(\cdot\mid c)}
  \left[Q_{0}(\cdot\mid z,c)\right].
\label{eq:frozen_teacher_centroid}
\end{equation}
The selected $z$ may shift as the student evolves, but the same frozen teacher
$p_{\theta_0}$ always interprets it. When the teacher evolves with the student, both
the selected $z$ and its interpretation depend on the student, and a
self-consistent solution satisfies
\begin{equation}
p_{\phi^{*}}(\cdot\mid c)
=\mathbb E_{z\sim\mu_{\phi^{*}}^{+}(\cdot\mid c)}
  \left[Q_{\phi^{*}}(\cdot\mid z,c)\right].
\label{eq:moving_teacher_fixed_point}
\end{equation}
This self-consistency condition does not establish convergence. A policy can
satisfy it without succeeding in the environment, because the policy being
trained determines both the selected information and its interpretation.
Direct imitation is the extreme case, with one-hot targets that give no
mass to alternative tokens.

\noindent\textbf{Low distillation loss need not mean task success.}
Forward KL measures agreement with the teacher and ignores task success.
It can also cover only the alternatives its teacher retains. In Online
ungated imitation, the zero-lag extreme in which the student imitates its own
generated tokens, the student's entropy on fixed prompts approaches zero as
drift from the initial policy grows. Teacher entropy was not logged, so these
observations show that the student collapses and leave open whether teacher
entropy collapse passes to the student.

\noindent\textbf{A KL anchor limits drift and supplies no episode-conditioned targets.}
Online RFT+KL limits some drift, but its success stays well below that of
ASCENT with the frozen privileged teacher. Its KL anchor keeps the student
closer to the base, whereas the privileged teacher also supplies
episode-conditioned soft targets.

\noindent\textbf{Collapsed behavior does not require a small adapter.}
At an endpoint, the collapsed direct imitation checkpoint has a lower
adapter stable rank and a next-token gradient about $25\times$ larger than the
frozen-teacher checkpoint. Online ungated imitation also ends with a larger adapter
norm than the frozen teacher, while its valid-action rate falls. Student entropy approaches zero, repeated actions become
more common, and executable actions become rare. These are behavioral and
endpoint associations: stable rank does not establish the cause of collapse,
and adapter norm or probe gradient neither measures the realized update nor
identifies a LoRA capacity limit.

\noindent\textbf{ASCENT changes the evidence and keeps the teacher fixed.}
ASCENT's targets vary with $z_i$ and $c_{i,t,j}$, while the teacher
$p_{\theta_0}$ that interprets them stays the same. Slowly evolving teachers can
hold their success and even improve late in the stream. None exceeds the frozen
teacher overall. Updating both models after every episode lets drift compound
as in Sec.~\ref{sec:direct_ntp}, now through the targets as well. Evolving the
teacher only in stages, after a verified gain of the student, could let the
supervision improve without this loop.

\subsection{Turn Efficiency under Success Composition}
\label{app:turn_efficiency}
In the ALFWorld seen-stream runs, failed episodes use all $T_{\max}=50$
turns, so the capped all-episode mean decomposes as
\begin{equation}
\mathbb E[T]
=\Pr(W)\,\mathbb E[T\mid W]
+\bigl(1-\Pr(W)\bigr)T_{\max},
\label{eq:turn_composition}
\end{equation}
where $W$ denotes episode success. A policy with higher success therefore has
fewer mean turns even if its successful trajectories are no shorter. To
compare turn use apart from this success composition, we pair each of the
3 ASCENT runs per scale in Table~\ref{tab:alfworld_seen} with the base on the
same task order.

Among successful episodes, the base and ASCENT average $17.9$ and $15.8$
turns at 4B, and $18.4$ and $12.6$ at 9B. These means still cover different
sets of solved tasks, so Table~\ref{tab:turn_efficiency} compares turn counts
only on tasks that both the un-evolved base model and ASCENT solve (shared
successes).

\begin{table}[H]
\centering
\small
\setlength{\tabcolsep}{4.5pt}
\renewcommand{\arraystretch}{1.08}
\begin{tabular}{@{}lrrrrrr@{}}
\toprule
\textbf{Scale} & \textbf{Both won} & \textbf{ASCENT only} &
\textbf{Base only} & \makecell{\textbf{Mean}\\$\Delta T$} &
\makecell{\textbf{Median}\\$\Delta T$} &
\makecell{\textbf{ASCENT uses}\\\textbf{fewer turns (\%)}} \\
\midrule
4B & $61.7${\scriptsize$\pm$1.2} & $35.7${\scriptsize$\pm$6.6} &
$3.3${\scriptsize$\pm$1.2} & $-4.8${\scriptsize$\pm$0.9} &
$-2.3${\scriptsize$\pm$0.5} & $63.8${\scriptsize$\pm$1.8} \\
9B & $74.3${\scriptsize$\pm$0.5} & $34.0${\scriptsize$\pm$2.2} &
$2.7${\scriptsize$\pm$0.5} & $-7.0${\scriptsize$\pm$0.3} &
$-3.5${\scriptsize$\pm$0.4} & $67.3${\scriptsize$\pm$4.7} \\
\bottomrule
\end{tabular}
\caption{\textbf{Shared-success turn comparison on ALFWorld seen.}
$\Delta T=T_{\mathrm{ASCENT}}-T_{\mathrm{Base}}$ on tasks solved by both
the un-evolved base model and ASCENT, so negative values favor ASCENT. The first
three columns count the tasks in the 140-task stream solved by both, by ASCENT
only, and by the base only. The last column is the percentage of shared successes
on which ASCENT uses fewer turns than the base. Entries are mean$\pm$std of
per-run statistics over the 3 paired runs per scale.}
\label{tab:turn_efficiency}
\end{table}

\noindent\textbf{ASCENT uses fewer turns on most shared successes.}
Table~\ref{tab:turn_efficiency} shows that, on these shared successes, the
mean and median $\Delta T$ are negative at both scales, and ASCENT uses fewer
turns than the base on most of these tasks. The paired comparison thus supports
an interaction-efficiency gain beyond the change in the success/failure mix.
It identifies neither the individual turns avoided nor the cause of
each shorter path.

\subsection{Retrieval and Execution for In-Context Adaptation}
\label{app:context_l1l2}

\begin{figure}[t]
\centering
\captionsetup[subfigure]{skip=0pt}
{\footnotesize
\textcolor[HTML]{B9C3D0}{\rule[.2ex]{1.3em}{.5em}}\, Qwen3.5-4B\quad
\textcolor[HTML]{475569}{\rule[.2ex]{1.3em}{.5em}}\, Qwen3.5-9B\par}
\vspace{2pt}
\begin{subfigure}[t]{0.49\linewidth}
\centering
\includegraphics[width=\linewidth]{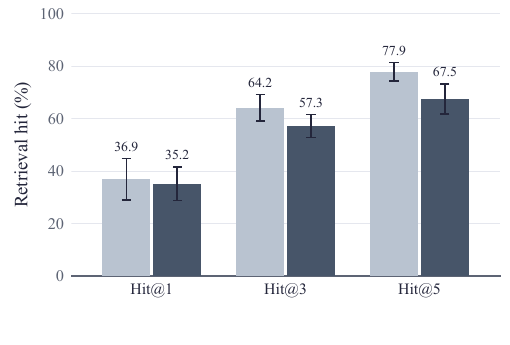}
\caption{Retrieval hit when a prior same-type success exists}
\end{subfigure}\hfill
\begin{subfigure}[t]{0.49\linewidth}
\centering
\includegraphics[width=\linewidth]{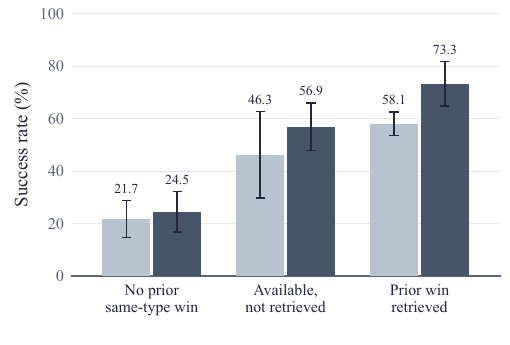}
\caption{Success conditioned on memory state}
\end{subfigure}
\caption{\textbf{Memory retrieval and success on ALFWorld seen.} For the
retrieval-based in-context adaptation baseline, bars show means with std
error bars: (a) retrieval of an available prior same-type success, and (b)
success by memory availability and retrieval.}
\label{fig:context_l1l2}
\end{figure}
We analyze a retrieval-based in-context adaptation baseline on three 140-task
ALFWorld seen streams per scale. Before task $i$, let $E_i=1$ when memory
holds an item from an earlier verified success of the same task type, and let
$R_i^{(k)}=1$ when at least one such item appears among the top-$k$ retrieved
memories. Fig.~\ref{fig:context_l1l2}(a) reports the retrieval hit rate
$\Pr(R_i^{(k)}=1\mid E_i=1)$. Fig.~\ref{fig:context_l1l2}(b) splits episode
success by whether a prior same-type success is absent ($E_i=0$), available
and missed by retrieval ($E_i=1,R_i^{(5)}=0$), or retrieved ($R_i^{(5)}=1$). Logged
and reconstructed retrieval counts agree for all 840 episodes.

\noindent\textbf{Retrieval often misses available experience and does not ensure success.}
As shown in Fig.~\ref{fig:context_l1l2}(a,b), the top five results often miss
an available prior same-type success, and retrieval is associated with higher
success. Retrieved episodes can still fail. Same task type is only a proxy for relevance,
so a retrieved item may not supply the instruction the current task needs. A
failure after retrieval therefore cannot by itself show that the agent failed
to execute useful guidance. The case study in Sec.~\ref{sec:case_study}
(Fig.~\ref{fig:g135_case}) examines concrete traces behind these aggregate
patterns. The paragraphs below record the provenance of each trace.

\noindent\textbf{Case-study protocol.}
The task in Fig.~\ref{fig:g135_case} is ALFWorld seen-stream episode
\texttt{game\_idx=135}: ``put a cool egg in microwave.'' Each method uses
Qwen3.5-4B and has one scored attempt with a shared 50-turn budget. Online
methods follow the same task order, whereas offline in-context adaptation
variants use fixed preconstructed banks. We keep these scored outcomes
separate from the later trace diagnostics shown in the figure.

\noindent\textbf{Prior in-context experience.}
Online MemP acquires a cool-then-place workflow at episode 138, so this
workflow is not yet available at episode 135. The retrieval baseline analyzed
above has 408 memories in its post-episode bank. Restricting the bank to
entries with \texttt{src\_idx<135} leaves 405 memories, and retrieval from
this pre-task bank returns the contradictory top item shown in
Fig.~\ref{fig:g135_case}.

\noindent\textbf{Prior parametric updates.}
Before episode 135, ASCENT updates on two verifier-accepted episodes from the
\texttt{pick\_cool\_then\_place\_in\_recep} family: episode 98 (cool a tomato
and place it in a garbage can) and episode 100 (cool a plate and place it in a
cabinet). The fast weights entering episode 135 may reflect these updates,
while episode 135 itself is scored before its own trajectory can update them.

\end{document}